\documentclass{article}
\usepackage[utf8]{inputenc}
\usepackage{amsmath} 
\usepackage{amssymb}  
\usepackage[ruled,vlined]{algorithm2e}
\usepackage{amsthm}
\usepackage{mathrsfs}
\usepackage{booktabs}
\usepackage{graphics} 
\usepackage{epsfig} 
\usepackage[textwidth=7in,textheight=9in]{geometry}
\usepackage[shortlabels]{enumitem}
\usepackage{authblk}

\usepackage{color}  
\usepackage{hyperref}
\hypersetup{
    colorlinks=true, 
    linktoc=all,     
    linkcolor=blue,  
    citecolor=blue
}

\newtheorem{assumption}{Assumption}
\newtheorem{theorem}{Theorem}

\newtheorem{remark}{Remark}
\newtheorem{lemma}{Lemma}

\title{Restless Bandit Problem with Rewards Generated by a Linear Gaussian Dynamical System}

\author[]{Jonathan Gornet and Bruno Sinopoli}
\affil[]{Department of Electrical and Systems Engineering, Washington University in Saint Louis}

\begin{document}

\maketitle

\begin{abstract}%
    Decision-making under uncertainty is a fundamental problem encountered frequently and can be formulated as a stochastic multi-armed bandit problem. In the problem, the learner interacts with an environment by choosing an action at each round, where a round is an instance of an interaction. In response, the environment reveals a reward, which is sampled from a stochastic process, to the learner. The goal of the learner is to maximize cumulative reward. In this work, we assume that the rewards are the inner product of an action vector and a state vector generated by a linear Gaussian dynamical system. To predict the reward for each action, we propose a method that takes a linear combination of previously observed rewards for predicting each action's next reward. We show that, regardless of the sequence of previous actions chosen, the reward sampled for any previously chosen action can be used for predicting another action's future reward, i.e. the reward sampled for action 1 at round $t-1$ can be used for predicting the reward for action $2$ at round $t$. This is accomplished by designing a modified Kalman filter with a matrix representation that can be learned for reward prediction. Numerical evaluations are carried out on a set of linear Gaussian dynamical systems and are compared with 2 other well-known stochastic multi-armed bandit algorithms. 
\end{abstract}


\section{Introduction}

The Stochastic Multi-Armed Bandit (SMAB) problem provides a rigorous framework for studying decision-making under uncertainty. The problem consists of the interaction between a learner and an environment for a set number of rounds. For each round, the learner chooses an action and in response the environment reveals a reward, which is sampled from a stochastic process, to the learner. The goal of the learner is to maximize cumulative reward. In the non-stationary case of the SMAB, the distributions of the reward for each action can change each round. A key result in the area is \cite{besbes2014stochastic} where it assumes that the cumulative changes in the reward distributions are bounded by a known constant. 

A more specific variation of the non-stationary SMAB are environments where the rewards are generated by $s$-step autoregressive models, i.e. an action's sampled reward $X_t$ is a linear combination of rewards $X_{t-s},\dots,X_{t-1}$ where $s$ is the autoregressive model order. Two key results that have tackled this SMAB environment are \cite{slivkins2008adapting,pmlr-v51-bogunovic16,chen2023non}. \cite{slivkins2008adapting} studied the performance of a number of algorithms for rewards generated by Brownian motion. In \cite{pmlr-v51-bogunovic16}, the authors consider when the rewards for each action is generated by a \textit{known} 1-step autoregressive process. In \cite{chen2023non}, they address SMAB environments modeled as an \textit{unknown} 1-step autoregressive or a \textit{known} $s$-step autoregressive. A key application of autoregressive models is presented in \cite{parker2020provably}, where the work tunes the hyperparameters, such as the gradient descent's learning rate, during the training process of a reinforcement learning based on neural networks. Finally, another perspective to $s$-step autoregressive models is \cite{gornet2022stochastic} where the reward $X_t$ and a context $\theta_t$ are generated by a Linear Gaussian Dynamical System (LGDS), where a context is a partial observation of the LGDS's state variables. The authors prove that a linear combination of previously observed contexts $\theta_{t-s}, \dots, \theta_{t-1}$ can be used to predict the reward $X_t$, a perspective similar to the environments considered in \cite{pmlr-v51-bogunovic16} and \cite{chen2023non}.

Our work proposes a discrete-time restless bandit with continuous state-space by assuming the state and rewards are generated by a LGDS. This paper extends the results in \cite{gornet2022stochastic} where now the context is no longer observed. The contributions of our paper are as follows. 


\noindent \textbf{Our Contributions:}
\begin{itemize}
    \item We introduce a SMAB environment where the rewards are generated by a LGDS in Section \ref{sec:Problem_Formulation}. 
    \item We prove that we can predict the reward for each action by using a linear combination of observed \textit{rewards}. For example, for an environment with 3 actions, if a learner chose action $A_{t-2} = 2$ at round $t-2$ and $A_{t-1} = 1$  at round $t-1$, the learner can take a linear combination of the sampled rewards $X_{t-2}$ and $X_{t-1}$ to predict the reward for action $a = 3$ at round $t$. The coefficients for the linear combination are from the identified modified Kalman filter matrix representation. We provide a proof of the error bound of the reward prediction for the identified modified Kalman filter. The idea is inspired by \cite{tsiamis2019finite} for identifying the Kalman filter, where now we assume that the measurements of the LGDS, a linear combination of the system's state variables, can change each round. (See Section \ref{sec:modified_kalman})
    \item Using the proved error bound of the reward prediction, we propose the algorithm Uncertainty-Based System Search (UBSS). The algorithm chooses the action that maximizes the sum of the reward prediction and its error. (See Section \ref{sec:algo}) 
    \item For numerical results in Section \ref{sec:numerical_results}, we apply UBSS to a parameterized LGDS to illustrate its numerical performance. Here, we compare UBSS to Upper Confidence Bound (UCB) algorithm  \cite{agrawal1995sample} and Sliding Window UCB (SW-UCB) \cite{garivier2008upper} algorithm, two well-known SMAB algorithms, and for which LGDS UBSS performs best. 
\end{itemize}


\noindent\textbf{Related Work}

\noindent One example of the non-stationary SMAB is the restless bandit where the reward for each action is the function of a state that is generated by a Markov chain \cite{whittle1988restless}. Whenever the learner chooses an action, the learner observes a Markov chain's state and a reward. This paper focuses on the case when the transition matrix of the Markov chain is unknown. Previous results in the discrete state-space Markov chain that use an approach similar to UCB are \cite{tekin2012online,ortner2012regret,wang2020restless,dai2011non,liu2011logarithmic}. \cite{jung2019regret} uses Thompson sampling, i.e. sampling parameters based on \textit{a priori} distribution of Markov chain, for action selection. We avoid comparisons with these previous results since the states of the Markov chain are discrete, whereas the results presented in this paper focus on when the states are continuous. This allows us to tackle a different set of application, such as hyperparameter optimization for reinforcement learning based on neural networks, e.g. \cite{parker2020provably}.

\section{Problem Formulation}\label{sec:Problem_Formulation}

The learner will interact with an environment modeled as a LGDS. We will consider the following LGDS:
\begin{equation}\label{eq:linear_dynamical_system}
    \begin{cases}
        z_{t+1} & = \Gamma z_t + \xi_t, ~~ z_0 \sim \mathcal{N}\left(\hat{z}_0,P_0\right) \\
        X_t & = \left\langle c_{A_t}, z_t \right\rangle +  \eta_t 
    \end{cases},
\end{equation}
where the reward $X_t \in \mathbb{R}$ is the inner product of an action vector $c_{A_t} \in \mathcal{A}$ and the state $z_t \in \mathbb{R}^d$. The process noise $\xi_t\in \mathbb{R}^d$ and measurement noise $\eta_t \in \mathbb{R}$ are independent normally distributed random variables, i.e. $\xi_t \sim \mathcal{N}(0,Q)$ and $\eta_t \sim \mathcal{N}(0,\sigma^2)$. The action vector $c_{A_t} \in \mathcal{A} = \{c_a \in \mathbb{R}^{d\times 1}\mid \left\Vert c_a \right\Vert_2 \leq B_c, a \in [k]\}$ where $B_c$ is known and $a \in [k] \triangleq \{1,2,\dots,k\}$ is the indexed action. Using similar notation as \cite{NIPS2011_e1d5be1c}, actions that are realized at round $t$ are denoted as $c_{A_t}\in \mathcal{A}$ and unrealized actions are denoted as $c_a \in \mathcal{A}$. We make the following assumptions on system \eqref{eq:linear_dynamical_system}.


\begin{assumption}\label{assum:stability}
    The state matrix $\Gamma$ is marginally stable, i.e. $\rho\left(\Gamma \right) \leq 1$. 
\end{assumption}

\begin{assumption}\label{assum:unknown}
    The vectors and matrices in system \eqref{eq:linear_dynamical_system} are unknown along with $Q$, $\sigma$, and $d$. However, number of actions $k$ is known. 
\end{assumption}

\begin{assumption}\label{assum:control_detect}
    The matrix pair $\left(\Gamma,Q^{1/2}\right)$ is controllable. The pair $\left(\Gamma,c_{a}^\top \right)$ is detectable for every vector $c_{a} \in \mathcal{A}$. 
\end{assumption}

The goal of the learner is to maximize the cumulative reward over a horizon $n > 0$, i.e. $\sum_{t=1}^n X_t$. The horizon length $n> 0$ may be unknown. To provide analysis on the performance of any proposed algorithm for maximizing cumulative reward in \eqref{eq:linear_dynamical_system}, regret is analyzed which is defined to be 
\begin{equation}\label{eq:regret}
    R_n \triangleq \sum_{t=1}^n \mathbb{E}\left[X_t^* - X_t\right], 
\end{equation}
where $X_t^*$ is the highest possible reward that can be sampled at round $t$. In the next section, we discuss a reward predictor for the LGDS \eqref{eq:linear_dynamical_system}. 

\section{Predicting the Reward of the LGDS}\label{sec:modified_kalman}

This section reviews the optimal $1$-step predictor of the rewards, in the mean-squared error sense, generated by LGDS \eqref{eq:linear_dynamical_system}: the Kalman filter. According to Assumption \ref{assum:unknown}, the matrices of the LGDS \eqref{eq:linear_dynamical_system} are unknown, implying that the Kalman filter needs to be identified. However, to the best of our knowledge, no current results exist for direct identification of the Kalman filter when the LGDS's \eqref{eq:linear_dynamical_system} action vector $c_{A_t} \in \mathcal{A}$ can change each round. Therefore, we propose a modified Kalman filter to identify. Imposing the assumptions posed in the previous section, we prove that prediction error of the modified Kalman filter is lower than or equal to the variance of the reward $X_t$, making it possible to extract a signal to predict the reward for each action. The added benefit of the modified Kalman filter is that it is tractable to identify. 


The Kalman filter uses the previous observations $X_1,\dots,X_t$ to compute an estimate of the state $z_t$ as $\hat{z}_{t} \triangleq \mathbb{E}\left[z_t \mid \mathcal{F}_{t-1}\right]$ where $\mathcal{F}_{t-1}$ is the sigma algebra generated by the rewards $X_1,\dots,X_{t-1}$,  
\begin{equation}\label{eq:Kalman_Filter_Reward}
    \begin{cases}
        \hat{z}_{t+1} & = \Gamma \hat{z}_t + \Gamma K_t \left(X_t - \left\langle c_{A_t}, \hat{z}_{t}\right\rangle \right), ~~ P_{t+1} = g\left(P_t,c_{A_t}\right) \\
        K_t & = P_t c_{A_t} \left(c_{A_t}^\top P_t c_{A_t} + \sigma\right)^{-1} \\
        \hat{X}_t & = \left\langle c_{A_t},\hat{z}_{t}\right\rangle 
    \end{cases},
\end{equation}
and $g\left(P,c\right)$ is defined to be the following Riccati equation \cite{gelb1974applied}
\begin{equation}\label{eq:g_definition}
    g\left(P,c\right) \triangleq \Gamma P \Gamma^\top + Q - \Gamma P c \left(c^\top P c + \sigma\right)^{-1} c^\top P \Gamma^\top .
\end{equation}

We impose the following assumption for the LGDS's \eqref{eq:linear_dynamical_system} initial state $z_0 \sim \mathcal{N}(\hat{z}_0,P_0)$ and the Kalman filter's \eqref{eq:Kalman_Filter_Reward} initial error covariance matrix $P_0$:
\begin{assumption}\label{assum:initial_state}
    The initial state $z_0 \in \mathbb{R}^d$ of the LGDS \eqref{eq:linear_dynamical_system} is sampled from a normal distribution with a mean $\hat{z}_0 \in \mathbb{R}^d$ that is a solution of $\hat{z}_0 = \Gamma \hat{z}_0$ and covariance matrix $P_0 \in \mathbb{R}^{d \times d}$. We assume that $P_0 = P_{\overline{a}}$, where $P_{\overline{a}}$ is the steady-state error covariance matrix, $P_{\overline{a}} = g\left(P_{\overline{a}},c_{\overline{a}}\right)$, $c_{\overline{a}} \in \mathcal{A}$. This assumption implies that the LGDS \eqref{eq:linear_dynamical_system} is in a steady-state distribution. 
\end{assumption}
\begin{remark}
    Assumption \ref{assum:initial_state} states that LGDS \eqref{eq:linear_dynamical_system} is in steady-state and the Kalman filter's \eqref{eq:Kalman_Filter_Reward} error covariance matrix is bounded. This is a reasonable assumption as the Kalman filter covariance matrix $P_t$ converges exponentially to the steady state covariance matrix $P_{\overline{a}}$ as $t$ increases if action $c_{\overline{a}} \in \mathcal{A}$ is consistently chosen. In addition, a similar assumption has been made in \cite{deistler1995consistency}, \cite{knudsen2001consistency}, and \cite{tsiamis2019finite}. Finally, it will be proven in Lemma \ref{lemma:Kalman_error_bound} that there exists an action $c_{\overline{a}} \in \mathcal{A}$ such that $P_{\overline{a}} \succeq P_t$ if $P_{\overline{a}} = P_0$. 
\end{remark}

As mentioned earlier, the parameters of LGDS \eqref{eq:linear_dynamical_system} are unknown due to Assumption \ref{assum:unknown}. Therefore, we propose to learn the Kalman filter \eqref{eq:Kalman_Filter_Reward} for reward prediction. However, since the Kalman filter matrices $P_t$ and $K_t$ change constantly, it is intractable to identify the Kalman filter. Therefore, we prove that there exists a modified Kalman filter that has a bounded reward prediction error regardless of the choices $c_{A_t} \in \mathcal{A}$ that is tractable to identify. For proving Theorem \ref{theorem:code_seq}, we first provide Lemma \ref{lemma:Kalman_error_bound} for the bound on the Kalman filter error covariance matrix $P_t$. 

\begin{lemma}\label{lemma:Kalman_error_bound}
    Let $P_a$, $a \in [k]$ be the steady state solution of the Kalman filter for each action $c_a \in \mathcal{A}$, $P_a = g\left(P_a,c_a\right)$, where $g\left(P_a,c_a\right)$ is defined in \eqref{eq:g_definition}. Define $P_{\overline{a}} \succeq 0$ to be the steady-state error covariance matrix of the Kalman filter \eqref{eq:Kalman_Filter_Reward} associated with action $c_{\overline{a}} \in \mathcal{A}$ such that $P_{\overline{a}} \succeq P_a$ for every action $a \in [k]$. By imposing Assumptions \ref{assum:stability}, \ref{assum:control_detect}, and \ref{assum:initial_state}, the LGDS \eqref{eq:linear_dynamical_system}, then $P_{\overline{a}} \succeq P_t$ for any $t = 1,2,\dots,n$.
\end{lemma}

Below is Theorem \ref{theorem:code_seq} which proves the existence of a modified Kalman filter with a bounded prediction error. Proof for Theorem \ref{theorem:code_seq} can be found in Appendix \ref{appendix:code_seq}. 


\begin{theorem}\label{theorem:code_seq}
    We define the following modified Kalman filter
    \begin{equation}\label{eq:modified_Kalman_filter}
        \begin{cases}
            \hat{z}_{t+1}' & = \Gamma \hat{z}_t' + \Gamma L_{A_t} \left(X_t - \left\langle c_{A_t}, \hat{z}_{t}'\right\rangle \right) \\
            X_t & = \left\langle c_{A_t},\hat{z}_{t}'\right\rangle + \gamma_{A_t}
        \end{cases}, ~~ L_{A_t} \triangleq P_{\overline{a}} c_{A_t}\left(c_{A_t}^\top P_{\overline{a}} c_{A_t} + \sigma\right)^{-1} , 
    \end{equation}
    where $\gamma_{A_t} \triangleq X_t - \left\langle c_{A_t},\hat{z}_{t}'\right\rangle \sim \mathcal{N}\left(0,c_{A_t}^\top P_t' c_{A_t} + \sigma\right)$ and $P_t' \triangleq \mathbb{E}\left[\left(z_t - \hat{z}_{t}'\right)\left(z_t - \hat{z}_{t}'\right)^\top \mid \mathcal{F}_{t-1}\right]$. It is proven for the modified Kalman filter \eqref{eq:modified_Kalman_filter} that 1) the matrix $\Gamma - \Gamma L_{A_t}c_{A_t}^\top$ is stable and 2) the variance of the residual $\mbox{Var}\left(\gamma_{A_t}\right)$ is bounded.
\end{theorem}

The key takeaway for Theorem \ref{theorem:code_seq} is that there exists a modified Kalman filter \eqref{eq:modified_Kalman_filter} that is easier to identify in comparison to the Kalman filter \eqref{eq:Kalman_Filter_Reward} at the expense of a higher prediction error $\mbox{Var}\left(\gamma_a\right) \geq c_a^\top P_t c_a + \sigma^2$. This is because the modified Kalman filter has only a finite number of gain matrices $L_{A_t}$ and a static covariance matrix $P_{\overline{a}}$. In addition, the variance of the prediction error $\mbox{Var}\left(\gamma_a\right)$ has an upper-bound. 


\subsection{Learning the Modified Kalman filter}\label{subsec:unique_ar}

Using Theorem \ref{theorem:code_seq} and inspired by the results presented in \cite{tsiamis2019finite}, we will learn the modified Kalman filter since the matrices and vectors in the LGDS \eqref{eq:linear_dynamical_system} and its modified Kalman filter \eqref{eq:modified_Kalman_filter} are unknown. Let parameter $s > 0$ denote how far in the past the learner will look. We define the tuple $\mathbf{c} \triangleq \begin{pmatrix}
    c_{A_{t-s}} & \dots & c_{A_{t-1}}
\end{pmatrix}$ as the sequence of actions chosen by the learner from rounds $t-s$ to $t-1$. The reward $X_t = \left\langle c_a, z_t \right\rangle + \eta_t$ for action $a \in [k]$ can be expressed as a linear combination of rewards $X_{t-s},\dots,X_{t-1}$ generated by the tuple $\mathbf{c}$ using the matrices defined in the modified Kalman filter \eqref{eq:modified_Kalman_filter}:
\begin{multline}
    X_t = c_{a}^\top \left(\Gamma - \Gamma L_{A_{t-1}} \right) \cdots \left(\Gamma - \Gamma L_{A_{t-s+1}} \right) \Gamma L_{A_{t-s}}X_{t-s} + \dots \\ + c_{a}^\top \Gamma L_{A_{t-1}}X_{t-1} + c_{a}^\top \left(\Gamma - \Gamma L_{A_{t-1}} c_{A_{t-1}}^\top\right) \cdots \left(\Gamma - \Gamma L_{A_{t-s}} c_{A_{t-s}}^\top\right) \hat{z}_{t-s}' + \gamma_a \nonumber, 
\end{multline}

Therefore, let there be defined the vectors $G_{c_a|\mathbf{c}}$, $c_a \in \mathcal{A}$, and $\Xi\left(\mathbf{c}\right)$ to express the reward $X_t = \left\langle c_a, z_t \right\rangle + \eta_t$:



\begin{equation}\label{eq:linear_model_one_action}
    \Rightarrow X_t = G_{c_a\mid\mathbf{c}}^\top\Xi_t\left(\mathbf{c}\right) + \beta_a + \gamma_a, 
\end{equation}
\begin{align}
    G_{c_{a}\mid\mathbf{c}} & \triangleq\begin{bmatrix}
        c_{a}^\top \left(\Gamma - \Gamma L_{A_{t-1}} \right) \cdots \left(\Gamma - \Gamma L_{A_{t-s+1}} \right) \Gamma L_{A_{t-s}} & \dots & c_{a}^\top \Gamma L_{A_{t-1}} \end{bmatrix} \in \mathbb{R}^{s \times 1} \nonumber \\
    \Xi_t\left(\mathbf{c}\right) & \triangleq \begin{bmatrix}
        X_{t-s} & \dots & X_{t-1} 
    \end{bmatrix}^\top \in \mathbb{R}^{s \times 1} \nonumber \\
    \beta_{a} & \triangleq c_{a}^\top \left(\Gamma - \Gamma L_{A_{t-1}} c_{A_{t-1}}^\top\right) \cdots \left(\Gamma - \Gamma L_{A_{t-s}} c_{A_{t-s}}^\top\right) \hat{z}_{t-s}' \in \mathbb{R} \nonumber.
\end{align}

Based on equation \eqref{eq:linear_model_one_action}, we can express the reward $X_t = \left\langle c_a, z_t \right\rangle + \eta_t$ for each action $c_a \in \mathcal{A}$ using $G_{c_a|\mathbf{c}}$, $c_a \in \mathcal{A}$, and $\Xi\left(\mathbf{c}\right)$ with the following linear model:
\begin{equation}\label{eq:linear_code_model}
    \begin{pmatrix}
        \left\langle c_1, z_t \right\rangle + \eta_t \\
        \left\langle c_2, z_t \right\rangle + \eta_t \\
        \vdots \\
        \left\langle c_k, z_t \right\rangle + \eta_t \\
    \end{pmatrix} = \begin{pmatrix}
        G_{c_{1}\mid\mathbf{c}}^\top \\
        G_{c_{2}\mid\mathbf{c}}^\top \\
        \vdots \\
        G_{c_{k}\mid\mathbf{c}}^\top
    \end{pmatrix} \Xi_t\left(\mathbf{c}\right) + \begin{pmatrix}
        \beta_{1} \\
        \beta_{2} \\
        \vdots \\
        \beta_{k}
    \end{pmatrix} + \begin{pmatrix}
        \gamma_{1} \\
        \gamma_{2} \\
        \vdots \\
        \gamma_{k}
    \end{pmatrix}. 
\end{equation}




The linear model \eqref{eq:linear_code_model} proves that we only need to identify $k^{s+1}$ vectors $G_{c_{a}\mid\mathbf{c}}$. Therefore, we can 1) identify $G_{c_{a}\mid\mathbf{c}}$ for each action $c_a \in \mathcal{A}$ and 2) predict the reward $X_t$ using inner product of the identified $G_{c_{a}\mid\mathbf{c}}$ and sequence of rewards $X_{t-s},\dots,X_{t-1}$. 


\begin{remark}
    In the linear model \eqref{eq:linear_code_model}, there is a parameter $s > 0$ which is the number of previous rewards $X_{t-s},\dots,X_{t-1}$ used for predicting the reward $X_t$. Parameter $s > 0$ impacts the magnitude of the term $\beta_a$ (which decreases exponentially as $s$ increases) and the number of linear models $G_{c_a\mid \mathbf{c}}$ to identify (which increases exponentially as $s$ increases). 
\end{remark}


The following are assumed about $G_{c_{a}\mid\mathbf{c}}$, $\gamma_{a}$, and $c_{a} \in \mathcal{A}$:

\begin{assumption}\label{assum:G_a_Assumption}
    There exists a known upper bound $B_G$ such that $\Vert G_{c_{a}\mid\mathbf{c}} \Vert_2 \leq B_G$ for all $a \in [k]$ which is a common assumption to use in SMAB problems \cite{lattimore2020bandit}.
\end{assumption}

\begin{assumption}\label{assum:noise_bound}
There exists a known constant $B_R > 0$ such that for any round $t>0$, we have:
\begin{align}
    \sqrt{\mbox{tr}\left(Z_t\right)} \leq B_R, ~ Z_t \triangleq \mathbb{E}\left[z_tz_t^\top\right], ~  \mbox{Var}\left(\gamma_{a}\right) \leq c_a^\top P_{\overline{a}} c_a + \sigma_\eta \leq B_R^2 \mbox{ for } c_a \in \mathcal{A} \nonumber,
\end{align}
where $Z_t$ (which has the iteration $Z_{t+1} = \Gamma Z_t \Gamma^\top + Q$) is the covariance of the LGDS's \eqref{eq:linear_dynamical_system} state $z_t$. Results in the area of non-stationary SMAB have made similar assumptions (see \cite{chen2023non}). 
\end{assumption}

To learn $G_{c_{a}\mid\mathbf{c}}$, assume that at time points $\mathcal{T}_{c_{a}\mid\mathbf{c}} = \left\{t_1,\dots,t_{N_{a}}\right\}$ ($N_a$ is the number of times action $c_a \in \mathcal{A}$ is chosen) the following tuple sequence $\begin{pmatrix}
    c_{A_{t_i-s}},\dots,c_{A_{t_i-1}}
\end{pmatrix} = \mathbf{c} \in \mathcal{A}^s \mbox{ for } t_i \in \left\{t_1,\dots,t_{N_{a}}\right\}$ and action $c_{A_t} = c_a \in \mathcal{A}$ are chosen. We have the following linear model
\begin{equation}\label{eq:linear_model}
    \mathbf{X}_{\mathcal{T}_{c_{a}\mid\mathbf{c}}} = G_{c_{a}\mid\mathbf{c}}^\top \mathbf{Z}_{\mathcal{T}_{c_{a}\mid\mathbf{c}}} + \mathbf{B}_{\mathcal{T}_{c_{a}\mid\mathbf{c}}} + \mathbf{E}_{\mathcal{T}_{c_{a}\mid\mathbf{c}}},
\end{equation}
\begin{align}
    \mathbf{X}_{\mathcal{T}_{c_{a}\mid\mathbf{c}}} & \triangleq \begin{bmatrix}
        X_{t_1} & \dots & X_{t_{N_{a}}} 
    \end{bmatrix} \in \mathbb{R}^{1 \times N_{a}}, & \mathbf{Z}_{\mathcal{T}_{c_{a}\mid\mathbf{c}}} & \triangleq \begin{bmatrix}
        \Xi_{t_1}\left(\mathbf{c}\right) & \dots & \Xi_{t_{N_{a}}}\left(\mathbf{c}\right)
    \end{bmatrix} \in \mathbb{R}^{s \times N_{a}}, \nonumber \\
    \mathbf{B}_{\mathcal{T}_{c_{a}\mid\mathbf{c}}} & \triangleq \begin{bmatrix}
        \beta_{A_{t_1}} & \dots & \beta_{A_{t_{N_{a}}}}
    \end{bmatrix} \in \mathbb{R}^{1 \times N_{a}},  & \mathbf{E}_{\mathcal{T}_{c_{a}\mid\mathbf{c}}} & \triangleq \begin{bmatrix}
        \gamma_{A_{t_1}} & \dots & \gamma_{A_{t_{N_{a}}}}
    \end{bmatrix}  \in \mathbb{R}^{1 \times N_{a}} \nonumber. 
\end{align}

The least squares estimate of $G_{c_{a}\mid\mathbf{c}}$ in \eqref{eq:linear_model} is 
\begin{align}
    \hat{G}_{c_{a}\mid\mathbf{c}}\left(\mathcal{T}_{c_{a}\mid\mathbf{c}}\right) & = \mathbf{X}_{\mathcal{T}_{c_{a}\mid\mathbf{c}}}\mathbf{Z}_{\mathcal{T}_{c_{a}\mid\mathbf{c}}}^\top V_{a}\left(\mathcal{T}_{c_{a}\mid\mathbf{c}}\right)^{-1} \label{eq:identify_2}\\
    V_{a}\left(\mathcal{T}_{c_{a}\mid\mathbf{c}}\right) & \triangleq \lambda I_s + \mathbf{Z}_{\mathcal{T}_{c_{a}\mid\mathbf{c}}} \mathbf{Z}_{\mathcal{T}_{c_{a}\mid\mathbf{c}}}^\top = \lambda I_s + \sum_{i=1}^{N_a} \Xi_{t_i}\left(\mathbf{c}\right)\Xi_{t_i}\left(\mathbf{c}\right)^\top \label{eq:V_matrix},
\end{align}
where $\lambda > 0$ is a regularization term. Since there are $k^{s}$ codes $\mathbf{c} \in \mathcal{A}^s$, then there are $k^{s+1}$ vectors $G_{c_{a}\mid\mathbf{c}}$ to learn. 



\section{Uncertainty-Based System Search Restless Bandit Problem}\label{sec:algo}

The section above provided a predictor, the modified Kalman filter, for the rewards generated by the LGDS \eqref{eq:linear_dynamical_system}. It also provided a methodology for identifying the predictor. Now that the reward can be predicted using an identified modified Kalman filter, we discuss how to use the predictor in Algorithm \ref{alg:cap_2}, Uncertainty-Based System Search (UBSS). The general scheme for UBSS is to 1) identify the predictor $G_{c_a\mid\mathbf{c}}$ for each action $c_a \in \mathcal{A}$ and 2) select actions that balances what the learner predicts will return the highest reward versus which actions the learner is the most uncertain due to the error of the predictor $\hat{G}_{c_a\mid\mathbf{c}}$. Therefore, for each round $t$ in UBSS, the learner will choose actions based on the following optimization problem 
\begin{equation}\label{eq:UCB_action_selection}
    \underset{a \in [k]}{\arg\max} ~\hat{G}_{c_{a}\mid\mathbf{c}}\left(\mathcal{T}_{c_{a}\mid\mathbf{c}}\right)^\top \Xi_t\left(\mathbf{c}\right) + \left(e_{c_{a} \mid \mathbf{c}}\left(\delta_e\right) + b_{c_{a}\mid \mathbf{c}}\left(\delta_b\right) \right) \sqrt{\Xi_t\left(\mathbf{c}\right)^\top V_{a}\left(\mathcal{T}_{c_{a}\mid \mathbf{c}}\right)^{-1}\Xi_t\left(\mathbf{c}\right)}, 
\end{equation}
where with a probability of at least $(1-\delta_e)(1-\delta_b)$, the following inequality is satisfied (see Theorem 9 in Appendix \ref{sec:perturb_error}):

\begin{equation}\label{eq:uncertainty}
    \hat{G}_{c_{a}\mid\mathbf{c}}\left(\mathcal{T}_{c_{a}\mid\mathbf{c}}\right)^\top \Xi_t\left(\mathbf{c}\right)- G_{c_{a}\mid\mathbf{c}}^\top \Xi_t\left(\mathbf{c}\right) \leq \left(e_{c_{a} \mid \mathbf{c}}\left(\delta_e\right) + b_{c_{a}\mid \mathbf{c}}\left(\delta_b\right) \right) \sqrt{\Xi_t\left(\mathbf{c}\right)^\top V_{a}\left(\mathcal{T}_{c_{a}\mid \mathbf{c}}\right)^{-1}\Xi_t\left(\mathbf{c}\right)}.
\end{equation}

The terms $e_{c_a \mid \mathbf{c}}\left(\delta_e\right)$ and $b_{c_a \mid \mathbf{c}}\left(\delta_b\right)$ are defined as
\begin{align}
    e_{c_a \mid \mathbf{c}}\left(\delta_e\right) & \triangleq \sqrt{2B_R^2\log\left(\frac{1}{\delta_e}\frac{ \det(V_a\left(\mathcal{T}_{c_{a}\mid\mathbf{c}}\right))^{1/2}}{\det(\lambda I)^{1/2}}\right)} \label{eq:e_probabilistic_bound} \\
    b_{c_a\mid \mathbf{c}}\left(\delta_b\right) & \triangleq \sqrt{N_a} \frac{B_c B_R}{\delta_b} \sqrt{\mbox{tr}\left(I-\lambda V_a\left(\mathcal{T}_{c_{a}\mid\mathbf{c}}\right)^{-1}\right)}
    +  \lambda \sqrt{\mbox{tr}\left(V_a\left(\mathcal{T}_{c_{a}\mid\mathbf{c}}\right)^{-1}\right)} B_G. \label{eq:b_probabilistic_bound}
\end{align}

Reward prediction uncertainty \eqref{eq:uncertainty} of action $c_a \in \mathcal{A}$ is impacted directly $V_{a}\left(\mathcal{T}_{c_{a}\mid \mathbf{c}}\right)$ which is a sum of $N_a$ (number of times action $c_a \in \mathcal{A}$ is chosen) positive semi-definite matrices. Therefore, choosing an action frequently (large $N_a$) will lower the reward prediction uncertainty. The rationale behind optimization problem \eqref{eq:UCB_action_selection} is to balance choosing the action with the highest reward versus the action with the most uncertainty. We summarize below which term is defined to be in \eqref{eq:UCB_action_selection} within Algorithm \ref{alg:cap_2} which consists of an \textbf{Exploitation term} (which action the learner expects to return the highest reward)
and an \textbf{Exploration term} (how much should the learner explore an action). 

\begin{itemize}
    \item \textbf{Exploitation term}: $\hat{G}_{c_a\mid\mathbf{c}}\left(\mathcal{T}_{c_a\mid\mathbf{c}}\right)^\top \Xi_t \left(\mathbf{c}\right)$
    \item \textbf{Exploration term}: $\left(e_{c_{a} \mid \mathbf{c}}\left(\delta_e\right) + b_{c_{a}\mid \mathbf{c}}\left(\delta_b\right) \right) \sqrt{\Xi_t\left(\mathbf{c}\right)^\top V_{a}\left(\mathcal{T}_{c_{a}\mid \mathbf{c}}\right)^{-1}\Xi_t\left(\mathbf{c}\right)}$
\end{itemize}

Parameters $\left(\delta_e,\delta_b\right)$ are failure rates of the bound in \eqref{eq:uncertainty} where $\left(\delta_e,\delta_b\right)$ values closer to 0 computes a larger bound \eqref{eq:uncertainty}. Parameter $s$ is the number of previously observed rewards $X_{t-\tau}$ ($\tau=1,\dots,s$) used for predicting the next reward. The number of models to learn increases exponentially as $s$ increases. Finally, $\lambda$ is a regularization parameter to ensure that \eqref{eq:V_matrix} is invertible.

\begin{algorithm}[h]
\SetAlgoLined
\textbf{Input: } $\delta_e,\delta_b \in (0,1)$, $\lambda > 0$, $s \in \mathbb{N}$, $B_c,B_R,B_G \in \mathbb{R}^+$ \\
\tcp{Initialization}
\For{$\mathbf{c} \in \left\{\begin{pmatrix}
    c_{a_1} & c_{a_2} & \dots & c_{a_s}
\end{pmatrix} \in \mathcal{A}^s\right\}$}{
\For{$a \in [k]$}{
$\mathcal{T}_{c_{a}\mid\mathbf{c}} \gets \{\},~~V_{a}\left(\mathcal{T}_{c_{a}\mid\mathbf{c}}\right) \gets \lambda I_s, ~~ \hat{G}_{c_{a}\mid\mathbf{c}}\left(\mathcal{T}_{c_{a}\mid\mathbf{c}}\right) \gets \mathbf{0}_{s \times 1}$ \\
$\left(e_{c_a \mid \mathbf{c}}\left(\delta_e\right),b_{c_a \mid \mathbf{c}}\left(\delta_b\right) \right) \gets 1/\epsilon$ where $\epsilon$ small \\
}
}
\tcp{Learner interaction with LGDS}
\For{$t = 1,2,\dots,n$}{
\eIf{ $t \geq s$ }{
\tcp{Action selection}
$A_t \gets \underset{a \in \{1,2,\dots,k\}}{\arg\max} \hat{G}_{c_a\mid\mathbf{c}}\left(\mathcal{T}_{c_a\mid\mathbf{c}}\right)^\top \Xi_t\left(\mathbf{c}\right) + \left(e_{c_{a} \mid \mathbf{c}}\left(\delta_e\right) + b_{c_{a}\mid \mathbf{c}}\left(\delta_b\right) \right) \sqrt{\Xi_t\left(\mathbf{c}\right)^\top V_{a}\left(\mathcal{T}_{c_{a}\mid \mathbf{c}}\right)^{-1}\Xi_t\left(\mathbf{c}\right)}$ \\
Sample $X_t$ from \eqref{eq:linear_dynamical_system} \\
$\mathbf{c} \gets\begin{pmatrix}
    c_{A_{t-s}} & \dots & c_{A_{t-1}}
\end{pmatrix}$  \\
$\mathcal{T}_{c_{A_t}\mid\mathbf{c}} \gets\mathcal{T}_{c_{A_t}\mid\mathbf{c}} \cup \{t\}$ \\
\tcp{Update estimates}
$V_{A_t}\left(\mathcal{T}_{c_{A_t}\mid\mathbf{c}}\right) \gets V_{A_t}\left(\mathcal{T}_{c_{A_t}\mid\mathbf{c}}\right) + \Xi_t\left(\mathbf{c}\right) \Xi_t\left(\mathbf{c}\right)^\top$ \\
$\hat{G}_{c_{A_t}\mid\mathbf{c}}\left(\mathcal{T}_{c_{A_t}\mid\mathbf{c}}\right) \gets \hat{G}_{c_{A_t}\mid\mathbf{c}}\left(\mathcal{T}_{c_{A_t}\mid\mathbf{c}}\right)^\top + X_t \Xi_t\left(\mathbf{c}\right)^\top V_{A_t}\left(\mathcal{T}_{c_{A_t}\mid\mathbf{c}}\right)^{-1}$ \\
\tcp{Update bounds}
Set $e_{c_{A_t} \mid \mathbf{c}}\left(\delta\right)$ and $b_{c_{A_t} \mid \mathbf{c}}\left(\delta\right)$ based on \eqref{eq:e_probabilistic_bound} and \eqref{eq:b_probabilistic_bound}, respectively.}
{
$A_t \gets \mbox{Sample uniformly } a \sim [k]$ \\
Sample $X_t$ from \eqref{eq:linear_dynamical_system} \\
}
}
\caption{Uncertainty-Based System Search (UBSS)}\label{alg:cap_2}
\end{algorithm}

\subsection{Regret Performance}\label{sec:regret_performance}

Algorithm \ref{alg:cap_2}, UBSS, has the following upper bound on regret \eqref{eq:regret}. Proof is in Appendix \ref{appendix:regret_performance}.


\begin{theorem}\label{theorem:algo_performance}    
Using Algorithm \ref{alg:cap_2} and setting $\delta_e=\delta_b = \delta \in (0,1)$, regret \eqref{eq:regret} satisfies the following inequality with a probability of at least $(1-\delta)^5$:
    \begin{equation}\label{eq:regret_bound_UCB}
        R_n \leq \sum_{t=1}^s \max_{c_a \in \mathcal{A}}\mathbb{E}\left[\left\langle c_{a^*} - c_a, z_t \right\rangle \right] + \sum_{a=1}^k 2(n-s)B_c^2 B_R^2 \left(1-\left(1-\delta\right)^4\left(1-\exp\left(\frac{-4B\left(\delta\mid\mathbf{c}\right)^2}{2\Delta G_{c_{a} \mid \mathbf{c}}^\top \Sigma_{\Xi_t\left(\mathbf{c}\right)}\Delta G_{c_{a} \mid \mathbf{c}}}\right) \right)\right),
    \end{equation}
    where $\Delta G_{c_{a} \mid \mathbf{c}} \triangleq G_{c_{a^*} \mid \mathbf{c}} - G_{c_{a} \mid \mathbf{c}}$, $\Sigma_{\Xi_t\left(\mathbf{c}\right)} \triangleq \mathbb{E}\left[\Xi_t\left(\mathbf{c}\right)\Xi_t\left(\mathbf{c}\right)^\top\right]$,
    and $B\left(\delta\mid\mathbf{c}\right)$ is 
    \begin{multline}\label{eq:B_def}
        B\left(\delta\mid\mathbf{c}\right) \triangleq \sqrt{2B_R^2 \log\left(\frac{1}{\delta}\frac{\left( s\lambda + (n-s) \frac{\mathbb{E}\left[\left\Vert\Xi_t\left(\mathbf{c}\right)\right\Vert_2\right]}{\delta}\right)^{s/2}}{\lambda^{s/2}}\right)} \sqrt{\frac{s}{\lambda}} \frac{\mathbb{E}\left[\left\Vert \Xi_t\left(\mathbf{c}\right) \right\Vert_2\right]}{\delta} \\ + 
       \sqrt{n-s} \frac{B_c B_R}{\delta} \sqrt{s}\sqrt{\frac{s}{\lambda}} \frac{\mathbb{E}\left[\left\Vert \Xi_t\left(\mathbf{c}\right) \right\Vert_2\right]}{\delta}  + \lambda B_G \sqrt{\frac{s}{\lambda}} \frac{\mathbb{E}\left[\left\Vert \Xi_t\left(\mathbf{c}\right) \right\Vert_2\right]}{\delta}. 
    \end{multline}
\end{theorem}

Theorem \ref{theorem:algo_performance} proves that regret increases at worst linearly, i.e. $\mathcal{O}\left(n\right)$. If the covariance for two different actions is large, e.g. large $\Delta G_{c_{a} \mid \mathbf{c}}$ and $\Sigma_{\Xi_t\left(\mathbf{c}\right)}$ terms, then the bound will decrease. The bound decreases exponentially as uncertainty \eqref{eq:uncertainty} decreases. 


\section{Numerical Results}\label{sec:numerical_results}

For numerical results, we generated rewards $X_t \in \mathbb{R}$ for each action $\{c_1,c_2\}$ from the following LGDS:
\begin{equation}\label{eq:numerical_LGDS}
    \begin{cases}
        z_{t+1} & = \begin{pmatrix}
            0.9 R\left(\theta\right) & I_2 \\
            \mathbf{0}_{2 \times 2} & 0.9 R\left(\theta\right)
        \end{pmatrix}z_t + \xi_t \\
        X_t & = \left\langle c_{A_t}, z_t \right\rangle + \eta_t
    \end{cases}, ~ \begin{cases} 
    R\left(\theta\right) & \triangleq \begin{pmatrix}
        \cos\theta & \sin\theta \\
        -\sin\theta & \cos\theta
    \end{pmatrix} \\
    c_1 & \triangleq \begin{pmatrix}
        10 & 0 & 0 & 0
    \end{pmatrix} \\
    c_2 & \triangleq \begin{pmatrix}
        0 & 10 & 0 & 0
    \end{pmatrix}
\end{cases}, 
\end{equation}
where the process noise $\xi_t \sim \mathcal{N}\left(0,I_4\right)$ and the measurement noise $\eta_t \sim \mathcal{N}\left(0,1\right)$ are sampled from standard normal distributions. The LGDS \eqref{eq:numerical_LGDS} is proposed to study how the magnitude of the error covariance matrix $P_{\overline{a}}$ impacts performance of UBSS, where $P_{\overline{a}}$ is directly impacted by the parameter $\theta \in [0,2\pi]$. Prior to the learner's interaction with the LGDS \eqref{eq:numerical_LGDS}, $10^4$ time steps are computed of the LGDS \eqref{eq:numerical_LGDS} to set the system to a steady state. After, the $10^4$ time steps, the length of the interaction between the environment and the learner is $n = 10^4$ rounds. Regret \eqref{eq:regret} is used to provide a metric of performance. Parameter $s$ in Algorithm \ref{alg:cap_2}, UBSS, is set to $1$ in the top left plot. For comparison, we consider UCB \cite{agrawal1995sample}, SW-UCB \cite{garivier2008upper}, and a learner that selects a random action each round (this learner is denoted as Random). We use UCB as a comparison since the eigenvalues of the LGDS \eqref{eq:numerical_LGDS} state matrix $\Gamma \in \mathbb{R}^{4 \times 4}$ is Schur, implying that the reward distributions have a bounded covariance with a mean of zero. SW-UCB is also used as a comparison since the reward is still generated by a dynamical system. Finally, Random is used as baseline for worst performance. 



\begin{figure}[h]
    \centering
    \includegraphics[width=0.9\linewidth]{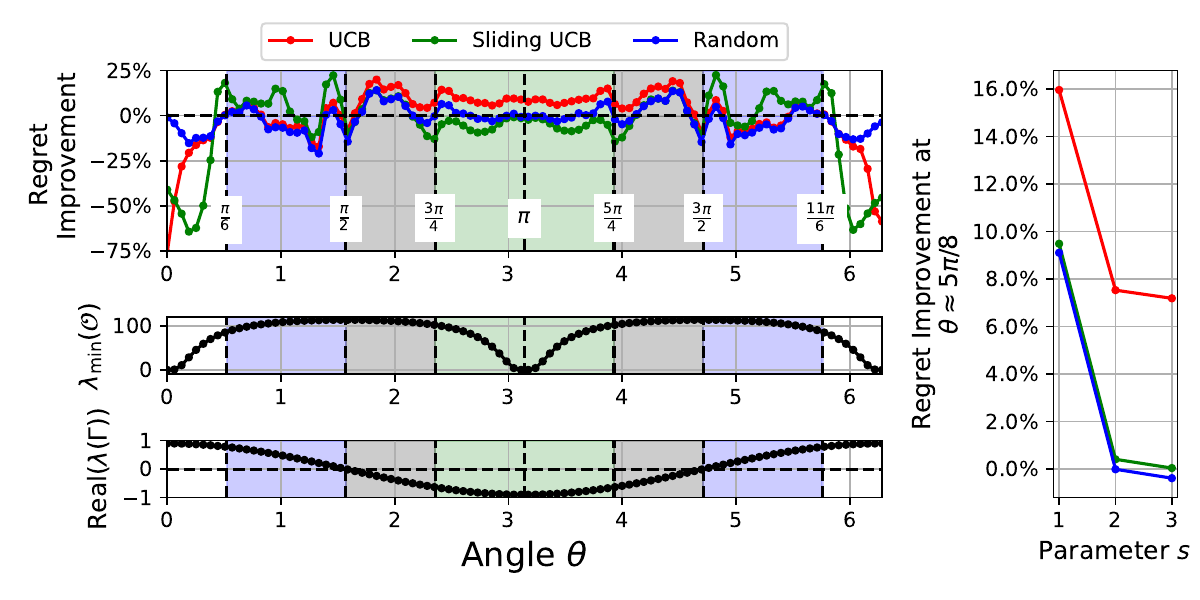}
    \caption{Comparison algorithm's regret normalized with respect to UBSS's regret. A positive percent implies that UBSS has a lower regret than the compared algorithm. }
    \label{figure:regret_best_case}
\end{figure}

In the top left plot of Figure \ref{figure:regret_best_case}, the percentage of UBSS's regret \eqref{eq:regret} is lower than UCB (red), Sliding UCB (green) and Random (blue) regret is shown for each $\theta \in [0,2\pi]$. The middle plot of Figure \ref{figure:regret_best_case} is the minimum eigenvalue of the Observability Gramian $\mathcal{O}$ \cite{hespanha2018linear} for both actions $c_a \in \{c_1,c_2\}$, which is the solution of the Lyapunov equation $\mathcal{O} = \Gamma^\top \mathcal{O} \Gamma + c_a c_a^\top$. The bottom plot of Figure \ref{figure:regret_best_case} is the real part of the eigenvalue of the state matrix $\Gamma$. In the white regions, all the comparison algorithms outperform UBSS. Based on the plot in the middle, it appears that the low Observability Gramian minimum eigenvalue and a positive real part of the state matrix's eigenvalue is the cause. For the blue regions, no algorithm outperforms Random, implying that the rewards are too noisy to estimate/predict for the compared algorithms. Finally, the gray regions is approximately where UBSS performs the best, providing approximately a 10\% improvement for each of the algorithms mid-region. Based on the bottom 2 plots, this increase in performance is from the high observability and an eigenvalue with a negative real part for the state matrix. High observability lowers the magnitude of the error covariance matrix $P_{\overline{a}}$, which leads to a lower regret bound of UBSS. In addition, an eigenvalue with a negative real part for the state matrix leads to rapid switching of the optimal action, making it difficult for UCB to adapt. For the plot on the far right, this is the relative performance of UBSS for each parameter $s = 1,2,3$ when the LGDS system \eqref{eq:numerical_LGDS} parameter set to approximately $5\pi/8$ (approximately where we see the largest improvement in performance of UBSS in the top left plot). Therefore, it appears that as $s$ increases to $s = 2,3$, regret performance of UBSS decreases. Since the number of parameters to identify increases exponentially as $s$ increases (leading to longer exploration times), regret performance of UBSS decreases as $s$ increases.


\section{Conclusion}

We have presented an algorithm for addressing a variation of the restless bandit with a continuous state-space. The rewards generated by this restless bandit variation is a LGDS. Based on the formulation, we propose to learn a representation of the modified Kalman filter to predict the rewards for each action. We have shown that regardless of the sequence of actions chosen, the learned representation of the modified Kalman filter converges. It is then proven what strategy should be used given the bound on regret, leading to an uncertainty-based strategy. 


In this work, we have not considered how the sequence of actions impact prediction error, how to choose window size (how far the learner looks into the past), and best obtainable performance of SMAB with LGDS environments. First, the perturbation added for exploration only considers error of the model and not the sequence of actions impact on the error of the prediction. In other words, the chosen sequence of actions are myopic. Therefore, future work will focus on the action sequence impact on the reward prediction. Next, an important parameter in UBSS is the window size. In UBSS, this is a parameter to set prior to the interaction with the environment. However, questions we care to ask is how to automate the process of choosing window size. Finally, UBSS has linear regret performance. Therefore, future work will be to derive the best obtainable performance of \textit{any} algorithm applied to a SMAB with rewards generated by this paper's proposed LGDS. We will then analyze if UBSS regret performance is close or far to the best obtainable performance. 


\newpage


\bibliographystyle{IEEEtran}
\bibliography{IEEEabrv,autosam}{}

\newpage
\appendix

\section{Proof of Lemma \ref{lemma:Kalman_error_bound}}\label{appendix:Kalman_error_bound}

\begin{proof}
    Let $P_{\overline{a}}$ be the steady-state solution of the Kalman filter \eqref{eq:Kalman_Filter_Reward} such that 
    \begin{equation}
        P_{\overline{a}} \succeq P_a \mbox{ for all } c_a \in \mathcal{A} \mbox{ where } P_a \mbox{ solves } P_a = g\left(P_a,c_a\right)\nonumber.
    \end{equation}

    Based on \cite{1333199}, if $P_{\overline{a}} \succeq P$ for any arbitrary $P \succeq 0$, then based on definition of $g\left(\cdot,\cdot\right)$ in \eqref{eq:Kalman_Filter_Reward}:
    \begin{equation}
        g\left(P_{\overline{a}},c_a\right) \succeq g\left(P,c_a\right) \nonumber,
    \end{equation}
    for any $c_a \in \mathcal{A}$ that is detectable with the LGDS \eqref{eq:linear_dynamical_system}. Assuming that $P_{\overline{a}} \succeq P$, we satisfy the following inequalities according to \cite{1333199}:
    \begin{align}
        P_{\overline{a}} \overset{(*)}{\succeq} g\left(P_{\overline{a}},c_{A_t}\right) & = \left(\Gamma - \Gamma L_{A_t}c_{A_t}^\top\right)P_{\overline{a}}\left(\Gamma - \Gamma L_{A_t}c_{A_t}^\top\right)^\top + Q + \sigma^2\Gamma L_{A_t}L_{A_t}^\top \Gamma^\top \nonumber \\
        & \succeq \left(\Gamma - \Gamma L_{A_t}c_{A_t}^\top\right)P\left(\Gamma - \Gamma L_{A_t}c_{A_t}^\top\right)^\top + Q + \sigma^2\Gamma L_{A_t}L_{A_t}^\top \Gamma^\top \nonumber. 
    \end{align}
    
    For $(*)$, assume for proof of contradiction that $g\left(P_{\overline{a}},c_{A_t}\right) \succeq P_{\overline{a}}$. Based on \cite{1333199}, then $g\left(g\left(P_{\overline{a}},c_{A_t}\right),c_{A_t}\right) \succeq g\left(P_{\overline{a}},c_{A_t}\right)$, where the iteration will either diverge to infinity or $P_{A_t}$. However, we assumed that $P_{\overline{a}} \succeq P_{A_t}$ and the steady-state solution $P_{A_t} = g\left(P_{A_t},c_{A_t}\right)$ exists from Assumption \ref{assum:control_detect}. Since there is a contradiction, $(*)$ is true. Therefore, based on above, let there be the modified Kalman filter \eqref{eq:modified_Kalman_filter} and the error covariance matrix $P_t' \triangleq \mathbb{E}\left[\left(z_t - \hat{z}_t'\right)\left(z_t - \hat{z}_t'\right)^\top\mid \mathcal{F}_{t-1}\right]$. The iteration of $P_t'$ is 
    \begin{equation}\label{eq:modified_iteration}
        P_{t+1}' = \left(\Gamma - \Gamma L_{A_t}c_{A_t}^\top\right)P_t'\left(\Gamma - \Gamma L_{A_t}c_{A_t}^\top\right)^\top + Q + \sigma^2\Gamma L_{A_t}L_{A_t}^\top \Gamma^\top , 
    \end{equation}
    where it is clear from arguments above that $P_{\overline{a}} \succeq P_t'$ and $P_{\overline{a}} \succeq P_{t+1}'$. 
\end{proof}

\section{Proof of Theorem \ref{theorem:code_seq}}\label{appendix:code_seq}

\begin{proof}
    The following is true based on Lemma \ref{lemma:Kalman_error_bound} and Assumption \ref{assum:control_detect}:
    \begin{itemize}
        \item The matrix pair $\left(\Gamma,Q^{1/2}\right)$ is controllable.
        \item The matrix $\sigma^2\Gamma L_{A_t}L_{A_t}^\top \Gamma^\top  \succeq 0$.
        \item The iteration \eqref{eq:modified_iteration} is bounded by $P_{\overline{a}}$ and is monotonic.
    \end{itemize}
    
    Therefore, the matrix $\Gamma - \Gamma L_{A_t}c_{A_t}^\top$ is stable. In addition, there exists an action $c_{A_t} \in \mathcal{A}$ such that the error $\gamma_{A_t} \sim \mathcal{N}\left(0,c_{A_t}^\top P_t' c_{A_t} + \sigma\right)$ has the following upper-bound 
    \begin{equation}
        c_{A_t}^\top P_t' c_{A_t} + \sigma \leq c_{A_t}^\top P_{\overline{a}} c_{A_t} + \sigma. \nonumber
    \end{equation}
\end{proof}

\section{Rationale behind OSS}\label{sec:Rationale_behind_OSS}

To develop an algorithm for LGDS restless bandits \eqref{eq:linear_dynamical_system}, we propose the following methodology. First, we will propose that actions should be chosen based on the following optimization problem 
\begin{equation}\label{eq:action_selection}
    \underset{a \in [k]}{\arg\max} ~\hat{G}_{c_{a}\mid\mathbf{c}}\left(\mathcal{T}_{c_{a}\mid\mathbf{c}}\right)^\top \Xi_t\left(\mathbf{c}\right) + u_a, 
\end{equation}
where $\hat{G}_{c_{a}\mid\mathbf{c}}\left(\mathcal{T}_{c_{a}\mid\mathbf{c}}\right)^\top \Xi_t\left(\mathbf{c}\right)$ is an \textbf{Exploitation term} (which action the learner predicts to have the highest reward) and an arbitrary term $u_a$ that is an \textbf{Exploration term} (how much should the learner continue exploring this action). We first prove the regret bound as a function of $u_a$. We then design $u_a$ such that the regret bound is minimized. 

In Subsection \ref{sec:model_error}, the model error bound of $\hat{G}_{c_{a}\mid\mathbf{c}}\left(\mathcal{T}_{c_{a}\mid\mathbf{c}}\right)$ is proven in Theorem \ref{theorem:model_error}. To prove Theorem \ref{theorem:model_error}, we provide Lemma \ref{lemma:bias_error}. In Subsection \ref{sec:perturb_error}, prediction error $\hat{G}_{c_{a}\mid\mathbf{c}}\left(\mathcal{T}_{c_{a}\mid\mathbf{c}}\right)^\top \Xi_t\left(\mathbf{c}\right)$ is proven. Theorem \ref{lemma:sigma_bound} proves the bound of the following difference:
\begin{equation}\label{eq:prediction_error}
    \hat{G}_{c_{a}\mid\mathbf{c}}\left(\mathcal{T}_{c_{a}\mid\mathbf{c}}\right)^\top \Xi_t\left(\mathbf{c}\right) - G_{c_{a}\mid\mathbf{c}}^\top \Xi_t\left(\mathbf{c}\right) . 
\end{equation}

Lemma \ref{lemma:perturb_bound} is then provided to prove the rate of convergence for prediction error that is derived in Theorem \ref{lemma:sigma_bound}. Finally, Subsection \ref{appendix:regret_performance} provides proofs of Algorithm \ref{alg:cap_2} regret bound. We first prove a regret bound if actions are chosen based on optimization problem \eqref{eq:action_selection}. Based on the derived bound, we propose what value $u_a$ should be set to in \eqref{eq:action_selection}. 

\subsection{Model Error}\label{sec:model_error}

The following lemma is provided which will be used for proving Theorem \ref{theorem:model_error}. 

\begin{lemma}\label{lemma:bias_error}
Let $\hat{G}_{c_{a}\mid\mathbf{c}}\left(\mathcal{T}_{c_{a}\mid\mathbf{c}}\right)$ and $G_{c_{a}\mid\mathbf{c}}$ be as in \eqref{eq:identify_2}
and \eqref{eq:linear_code_model}, respectively. Then, the matrix of the bias terms $\mathbf{B}_{\mathcal{T}_{c_{A_t}\mid\mathbf{c}}}$ satisfies the following inequality with a probability of at least $1-\delta_b$ (the failure probability value $\delta_b$ is assumed to be in the open interval $(0,1)$): 
\begin{equation}\label{eq:bias_bound_matrix}
        \left\Vert\mathbf{Z}_{\mathcal{T}_{c_{A_t}\mid\mathbf{c}}}\mathbf{B}_{\mathcal{T}_{c_{A_t}\mid\mathbf{c}}}^\top \right\Vert_{V_a\left(\mathcal{T}_{c_{a}\mid\mathbf{c}}\right)^{-1}} \leq  \sqrt{N_a} \frac{B_c B_R}{\delta_b} \sqrt{\mbox{tr}\left(I - \lambda V_a\left(\mathcal{T}_{c_{a}\mid\mathbf{c}}\right)^{-1}\right)}.
\end{equation}

\end{lemma}

\begin{proof}
    According to \eqref{eq:linear_model}, $\left\Vert\mathbf{Z}_{\mathcal{T}_{c_{A_t}\mid\mathbf{c}}}\mathbf{B}_{\mathcal{T}_{c_{A_t}\mid\mathbf{c}}}^\top \right\Vert_{V_a\left(\mathcal{T}_{c_{a}\mid\mathbf{c}}\right)^{-1}}$ can be expressed as
    \begin{equation}\label{eq:B_T_norm}
        \left\Vert\mathbf{Z}_{\mathcal{T}_{c_{A_t}\mid\mathbf{c}}}\mathbf{B}_{\mathcal{T}_{c_{A_t}\mid\mathbf{c}}}^\top \right\Vert_{V_a\left(\mathcal{T}_{c_{a}\mid\mathbf{c}}\right)^{-1}} = \left\Vert \sum_{t \in \mathcal{T}_{c_{a}\mid\mathbf{c}}} \Xi_t\left(\mathbf{c}\right) \beta_{A_t} \right\Vert_{V_a\left(\mathcal{T}_{c_{a}\mid\mathbf{c}}\right)^{-1}} .
    \end{equation}
    
    Based on the triangle and Cauchy-Schwarz inequalities, \eqref{eq:B_T_norm} has the following bound

    \begin{align}
        \left\Vert\mathbf{Z}_{\mathcal{T}_{c_{A_t}\mid\mathbf{c}}}\mathbf{B}_{\mathcal{T}_{c_{A_t}\mid\mathbf{c}}}^\top \right\Vert_{V_a\left(\mathcal{T}_{c_{a}\mid\mathbf{c}}\right)^{-1}} & = \left\Vert \Xi_{t_1}\left(\mathbf{c}\right)^\top \beta_{a;t_1} V_a\left(\mathcal{T}_{c_{a}\mid\mathbf{c}}\right)^{-1/2} + \dots + \Xi_{t_{N_a}}\left(\mathbf{c}\right)^\top \beta_{a;t_{N_a}} V_a\left(\mathcal{T}_{c_{a}\mid\mathbf{c}}\right)^{-1/2}\right\Vert_2  \nonumber \\
        & \overset{(*)}{\leq} \left\Vert \Xi_{t_1}\left(\mathbf{c}\right)^\top \beta_{a;t_1} V_a\left(\mathcal{T}_{c_{a}\mid\mathbf{c}}\right)^{-1/2} \right\Vert_2 + \dots \nonumber \\
        & ~~~~~~~~~~~~~~~~~~~~~~~~ + \left\Vert\Xi_{t_{N_a}}\left(\mathbf{c}\right)^\top \beta_{a;t_{N_a}} V_a\left(\mathcal{T}_{c_{a}\mid\mathbf{c}}\right)^{-1/2}\right\Vert_2 \nonumber \\
        & \overset{(**)}{\leq} \left\vert \beta_{a;t_1} \right\vert \left\Vert \Xi_{t_1}\left(\mathbf{c}\right)^\top V_a\left(\mathcal{T}_{c_{a}\mid\mathbf{c}}\right)^{-1/2} \right\Vert_2 + \dots \nonumber \\
        & ~~~~~~~~~~~~~~~~~~~~~~~~ + \left\vert \beta_{a;t_{N_a}} \right\vert \left\Vert\Xi_{t_{N_a}}\left(\mathbf{c}\right)^\top V_a\left(\mathcal{T}_{c_{a}\mid\mathbf{c}}\right)^{-1/2}\right\Vert_2 \nonumber,
    \end{align}
    \begin{equation}
        \Rightarrow \left\Vert\mathbf{Z}_{\mathcal{T}_{c_{A_t}\mid\mathbf{c}}}\mathbf{B}_{\mathcal{T}_{c_{A_t}\mid\mathbf{c}}}^\top \right\Vert_{V_a\left(\mathcal{T}_{c_{a}\mid\mathbf{c}}\right)^{-1}} \leq \sum_{i=1}^{N_a} \left\vert \beta_{a;t_{i}} \right\vert \left\Vert\Xi_{t_{i}}\left(\mathbf{c}\right)^\top V_a\left(\mathcal{T}_{c_{a}\mid\mathbf{c}}\right)^{-1/2}\right\Vert_2 \label{eq:matrix_V_bound},
    \end{equation}
    where $(*)$ uses the triangle inequality and $(**)$ uses the Cauchy-Schwarz inequality. The term $\left\vert\beta_{A_t} \right\vert$, $t \in \mathcal{T}_{c_{a}\mid\mathbf{c}}$ has the following bound using Cauchy-Schwarz inequality
    \begin{align}
        \left\vert\beta_{A_t} \right\vert & = \left\vert \left\langle c_a, \left(\Gamma - \Gamma L_{a_{t-1}} c_{a_{t-1}}^\top\right) \cdots \left(\Gamma - \Gamma L_{a_{t-s}} c_{a_{t-s}}^\top\right) \hat{z}_{t-s}\right\rangle\right\vert \nonumber \\
        & \leq \left\Vert c_a \right\Vert_2 \left\Vert \left(\Gamma - \Gamma L_{a_{t-1}} c_{a_{t-1}}^\top\right) \cdots \left(\Gamma - \Gamma L_{a_{t-s}} c_{a_{t-s}}^\top\right)\hat{z}_{t-s}\right\Vert_2 \label{eq:inequality_bias_1}. 
    \end{align}

    We can upper-bound the right side of \eqref{eq:inequality_bias_1} with $\left\Vert \hat{z}_t\right\Vert_2$. First, define $\hat{Z}_t \triangleq \mathbb{E}\left[\hat{z}_t \hat{z}_t^\top\mid \mathcal{F}_{t-1} \right]$ to get the following two iteration
    \begin{align}
        \hat{Z}_{t-s+1}^{[1]} & = \left(\Gamma - \Gamma L_{a_{t-s}} c_{a_{t-s}}^\top\right)\hat{Z}_{t-s}^{[1]} \left(\Gamma - \Gamma L_{a_{t-s}} c_{a_{t-s}}^\top\right)^\top + \mathbb{E}\left[X_{t-s}^2\right] \Gamma L_{a_{t-s}} L_{a_{t-s}}^\top \Gamma^\top \nonumber \\
        \hat{Z}_{t-s+1}^{[2]} & = \left(\Gamma - \Gamma L_{a_{t-s}} c_{a_{t-s}}^\top\right)\hat{Z}_{t-s}^{[2]} \left(\Gamma - \Gamma L_{a_{t-s}} c_{a_{t-s}}^\top\right)^\top \nonumber
    \end{align}
    where $\sqrt{\mbox{tr}\left(\hat{Z}_{t}^{[1]}\right)}$ and $\sqrt{\mbox{tr}\left(\hat{Z}_{t}^{[2]}\right)}$ are equal to 
    \begin{align}
        \sqrt{\mbox{tr}\left(\hat{Z}_{t}^{[1]}\right)} & = \left\Vert \hat{z}_t \right\Vert_2 \label{eq:first_iteration} \\
        \sqrt{\mbox{tr}\left(\hat{Z}_{t}^{[2]}\right)} & = \left\Vert \left(\Gamma - \Gamma L_{a_{t-1}} c_{a_{t-1}}^\top\right) \cdots \left(\Gamma - \Gamma L_{a_{t-s}} c_{a_{t-s}}^\top\right)\hat{z}_{t-s} \right\Vert_2 \label{eq:second_iteration}.
    \end{align}
    
    We know that $\Gamma - \Gamma L_{a_{t}} c_{a_{t}}^\top$ is stable for any $c_{A_t} \in \mathcal{A}$ based on Theorem \ref{theorem:code_seq}. Therefore, based on \eqref{eq:first_iteration} and \eqref{eq:second_iteration}, $\sqrt{\mbox{tr}\left(\hat{Z}_{t}^{[1]}\right)} \succeq \sqrt{\mbox{tr}\left(\hat{Z}_{t}^{[2]}\right)}$ implying that
    \begin{equation}
        \left\Vert \left(\Gamma - \Gamma L_{a_{t-1}} c_{a_{t-1}}^\top\right) \cdots \left(\Gamma - \Gamma L_{a_{t-s}} c_{a_{t-s}}^\top\right)\hat{z}_{t-s}\right\Vert_2 \leq \sqrt{\mbox{tr}\left(\hat{Z}_t\right)} = \left\Vert \hat{z}_t\right\Vert_2 \nonumber, 
    \end{equation}

    To bound $\left\Vert \hat{z}_t\right\Vert_2$, assumption \ref{assum:noise_bound} states that $\sqrt{\mbox{tr}\left(\mathbb{E}\left[z_tz_t^\top \right]\right)} \leq B_R$. Since $z_t = \hat{z}_t + e_t$ and by the orthogonality principle 
    \begin{equation}
        \mbox{tr}\left(\mathbb{E}\left[\hat{z}_te_t^\top + e_t\hat{z}_t^\top \right]\right) = 0, \nonumber 
    \end{equation}
    then using $Z_t$ in \eqref{eq:bias_bound_matrix},
    \begin{align}
        Z_t = \mathbb{E}\left[z_t z_t^\top \right] & = \mathbb{E}\left[\left(\hat{z}_t + e_t\right)\left(\hat{z}_t + e_t\right)^\top \mid \mathcal{F}_{t-1}\right] \nonumber \\
        & = \mathbb{E}\left[\hat{z}_t\hat{z}_t^\top + \hat{z}_t e_t^\top + e_t\hat{z}_t^\top + e_te_t^\top \mid \mathcal{F}_{t-1}\right] \nonumber\\
        & = \mathbb{E}\left[\hat{z}_t\hat{z}_t^\top \mid \mathcal{F}_{t-1}\right] + \mathbb{E}\left[\hat{z}_t e_t^\top + e_t\hat{z}_t^\top \mid \mathcal{F}_{t-1}\right]  + \mathbb{E}\left[e_te_t^\top \mid \mathcal{F}_{t-1}\right]  \nonumber\\
        & = \mathbb{E}\left[\hat{z}_t\hat{z}_t^\top \mid \mathcal{F}_{t-1}\right]  + \mathbb{E}\left[e_te_t^\top \mid \mathcal{F}_{t-1}\right]  \nonumber \\
        & = \hat{Z}_t + P_t \nonumber,
    \end{align}
    where $\hat{Z}_t \triangleq \mathbb{E}\left[z_tz_t^\top \mid \mathcal{F}_{t-1} \right]$ and $P_t \triangleq \mathbb{E}\left[e_te_t^\top \mid \mathcal{F}_{t-1} \right]$. Since $\hat{Z}_t \succeq 0$ and $P_t \succeq 0$, then 
    \begin{align}
        \sqrt{\mbox{tr}\left(Z_t\right)} = \sqrt{\mbox{tr}\left(\hat{Z}_t + P_t\right)} \geq \sqrt{\mbox{tr}\left(\hat{Z}_t\right)} \nonumber \\
        \Rightarrow \sqrt{\mbox{tr}\left(\hat{Z}_t\right)} \leq B_R\nonumber.
    \end{align}

    Therefore, using the Markov inequality \cite{boucheron2013concentration}, with a probability of at least $1-\delta_b$,
    \begin{equation}\label{eq: beta_bound}
        \left\vert\beta_{A_t} \right\vert \leq \left\Vert c_a \right\Vert_2 \frac{\sqrt{\mbox{tr}\left(\mathbb{E}\left[z_tz_t^\top\right]\right)}}{\delta_b} \leq \left\Vert c_a \right\Vert_2 \frac{B_R}{\delta_b} \leq \frac{B_c B_R}{\delta_b}.
    \end{equation}

    Using inequality \eqref{eq: beta_bound}, inequality \eqref{eq:matrix_V_bound} can be rewritten as
    \begin{align}
        \left\Vert\mathbf{Z}_{\mathcal{T}_{c_{A_t}\mid\mathbf{c}}}\mathbf{B}_{\mathcal{T}_{c_{A_t}\mid\mathbf{c}}}^\top \right\Vert_{V_a\left(\mathcal{T}_{c_{a}\mid\mathbf{c}}\right)^{-1}} & \leq \sum_{i=1}^{N_a} \frac{B_c B_R}{\delta_b} \left\Vert\Xi_{t_{i}}\left(\mathbf{c}\right)^\top V_a\left(\mathcal{T}_{c_{a}\mid\mathbf{c}}\right)^{-1/2}\right\Vert_2 \nonumber \\
        & = \sum_{i=1}^{N_a} \frac{B_c B_R}{\delta_b} \sqrt{\mbox{tr}\left(V_a\left(\mathcal{T}_{c_{a}\mid\mathbf{c}}\right)^{-1}\Xi_{t_{i}}\left(\mathbf{c}\right)\Xi_{t_{i}}\left(\mathbf{c}\right)^\top \right)}\nonumber \\
        & \leq \sum_{i=1}^{N_a} \frac{B_c B_R}{\delta_b} \sqrt{\mbox{tr}\left(V_a\left(\mathcal{T}_{c_{a}\mid\mathbf{c}}\right)^{-1}\Xi_{t_{i}}\left(\mathbf{c}\right)\Xi_{t_{i}}\left(\mathbf{c}\right)^\top \right) \cdot 1}\nonumber \\
        & \overset{(*)}{\leq} \frac{B_c B_R}{\delta_b} \sqrt{\sum_{i=1}^{N_a} 1} \sqrt{\sum_{i=1}^{N_a}\mbox{tr}\left(V_a\left(\mathcal{T}_{c_{a}\mid\mathbf{c}}\right)^{-1}\Xi_{t_{i}}\left(\mathbf{c}\right)\Xi_{t_{i}}\left(\mathbf{c}\right)^\top \right)}\nonumber \\
        & = \sqrt{N_a} \frac{B_c B_R}{\delta_b} \sqrt{\mbox{tr}\left(V_a\left(\mathcal{T}_{c_{a}\mid\mathbf{c}}\right)^{-1}\sum_{i=1}^{N_a}  \Xi_{t_{i}}\left(\mathbf{c}\right)\Xi_{t_{i}}\left(\mathbf{c}\right)^\top \right)}\nonumber \\
        & = \sqrt{N_a} \frac{B_c B_R}{\delta_b} \sqrt{\mbox{tr}\left(V_a\left(\mathcal{T}_{c_{a}\mid\mathbf{c}}\right)^{-1}\left(V_a\left(\mathcal{T}_{c_{a}\mid\mathbf{c}}\right) - \lambda I\right)\right)}\nonumber,
    \end{align}
    \begin{equation}
        \Rightarrow \left\Vert\mathbf{Z}_{\mathcal{T}_{c_{A_t}\mid\mathbf{c}}}\mathbf{B}_{\mathcal{T}_{c_{A_t}\mid\mathbf{c}}}^\top \right\Vert_{V_a\left(\mathcal{T}_{c_{a}\mid\mathbf{c}}\right)^{-1}}  \leq \sqrt{N_a} \frac{B_c B_R}{\delta_b} \sqrt{\mbox{tr}\left(I - \lambda V_a\left(\mathcal{T}_{c_{a}\mid\mathbf{c}}\right)^{-1}\right)}, \nonumber
    \end{equation}
    where $(*)$ uses the Cauchy-Schwarz inequality. Since $V_a\left(\mathcal{T}_{c_{a}\mid\mathbf{c}}\right) \succeq \lambda I$ based on \eqref{eq:V_matrix}, then the difference $I - \lambda V_a\left(\mathcal{T}_{c_{a}\mid\mathbf{c}}\right) \succeq 0$. Therefore, no further effort is needed to show that inequality \eqref{eq:bias_bound_matrix} is satisfied with a probability of at least $1-\delta_b$.
    
\end{proof}


Using Lemma \ref{lemma:bias_error}, Theorem \ref{theorem:model_error} can be proven.


\begin{theorem}\label{theorem:model_error}
    Let $\hat{G}_a\left(\mathcal{T}_{c_{a}\mid\mathbf{c}}\right)$ be identified based on \eqref{eq:identify_2} where the reward $\mathbf{X}_{\mathcal{T}_{c_{a}\mid\mathbf{c}}}$ has expression \eqref{eq:linear_model}. Then, the following inequality is satisfied with a probability of at least $\left(1-\delta_e\right)\left(1-\delta_b\right)$, where $\delta_e,\delta_b \in (0,1)$ are designed parameters for tolerated failure probabilities:
    \begin{equation}\label{eq:probabilistic_bound}
        \left\Vert \hat{G}_a\left(\mathcal{T}_{c_{a}\mid\mathbf{c}}\right) - G_{c_{a}\mid\mathbf{c}} \right\Vert_{V_a\left(\mathcal{T}_{c_{a}\mid\mathbf{c}}\right)} \leq e_{c_a \mid \mathbf{c}}\left(\delta_e\right) + b_{c_a \mid \mathbf{c}}\left(\delta_b\right),
    \end{equation}
    where $e_{c_a \mid \mathbf{c}}\left(\delta_e\right)$ and $b_{c_a\mid \mathbf{c}}\left(\delta_b\right)$ are defined in \eqref{eq:e_probabilistic_bound} and \eqref{eq:b_probabilistic_bound}, respectively. 
\end{theorem}


\begin{proof}
    Let $\mathbf{X}_{\mathcal{T}_{c_{a}\mid\mathbf{c}}}$ be expressed using \eqref{eq:linear_model}. The norm $\left\Vert\hat{G}_{c_{a}\mid\mathbf{c}}\left(\mathcal{T}_{c_{a}\mid\mathbf{c}}\right) - G_{c_{a}\mid\mathbf{c}} \right\Vert_2$ can be expressed the following way 
    \begin{equation}
         \hat{G}_{c_{a}\mid\mathbf{c}}\left(\mathcal{T}_{c_{a}\mid\mathbf{c}}\right) - G_{c_{a}\mid\mathbf{c}}  = - \lambda G_{c_{a}\mid\mathbf{c}} V_a\left(\mathcal{T}_{c_{a}\mid\mathbf{c}}\right)^{-1} + 
        \left(\mathbf{B}_{\mathcal{T}_{c_{A_t}\mid\mathbf{c}}} + \mathbf{E}_{\mathcal{T}_{c_{A_t}\mid\mathbf{c}}} \right) \mathbf{Z}_{\mathcal{T}_{c_{A_t}\mid\mathbf{c}}}^\top V_a\left(\mathcal{T}_{c_{a}\mid\mathbf{c}}\right)^{-1}, \nonumber
    \end{equation}
    which implies that:
    \begin{align}
        \left(\hat{G}_{c_{a}\mid\mathbf{c}}\left(\mathcal{T}_{c_{a}\mid\mathbf{c}}\right)^\top  - G_{c_{a}\mid\mathbf{c}}^\top \right)V_a\left(\mathcal{T}_{c_{a}\mid\mathbf{c}}\right) & = - \lambda G_{c_{a}\mid\mathbf{c}}^\top + 
        \left(\mathbf{B}_{\mathcal{T}_{c_{A_t}\mid\mathbf{c}}} + \mathbf{E}_{\mathcal{T}_{c_{A_t}\mid\mathbf{c}}} \right) \mathbf{Z}_{\mathcal{T}_{c_{A_t}\mid\mathbf{c}}}^\top , \nonumber \\
        \left(\hat{G}_{c_{a}\mid\mathbf{c}}\left(\mathcal{T}_{c_{a}\mid\mathbf{c}}\right)  - G_{c_{a}\mid\mathbf{c}} \right)^\top V_a\left(\mathcal{T}_{c_{a}\mid\mathbf{c}}\right) & = \left(- \lambda G_{c_{a}\mid\mathbf{c}} + 
        \mathbf{Z}_{\mathcal{T}_{c_{A_t}\mid\mathbf{c}}}\left(\mathbf{B}_{\mathcal{T}_{c_{A_t}\mid\mathbf{c}}} + \mathbf{E}_{\mathcal{T}_{c_{A_t}\mid\mathbf{c}}} \right)^\top \right)^\top , \nonumber         
    \end{align}
    and consequently:
    \begin{equation}\label{eq:triangle_inequality}
        \left\Vert\hat{G}_{c_{a}\mid\mathbf{c}}\left(\mathcal{T}_{c_{a}\mid\mathbf{c}}\right) - G_{c_{a}\mid\mathbf{c}}\right\Vert_{V_a\left(\mathcal{T}_{c_{a}\mid\mathbf{c}}\right)} = \left\Vert- \lambda G_{c_{a}\mid\mathbf{c}} + \mathbf{Z}_{\mathcal{T}_{c_{A_t}\mid\mathbf{c}}}\left(\mathbf{B}_{\mathcal{T}_{c_{A_t}\mid\mathbf{c}}} + \mathbf{E}_{\mathcal{T}_{c_{A_t}\mid\mathbf{c}}} \right)^\top \right\Vert_{V_a\left(\mathcal{T}_{c_{a}\mid\mathbf{c}}\right)^{-1}} \nonumber. 
    \end{equation}

    Using the triangle inequality, it follows from \eqref{eq:triangle_inequality} that:
    \begin{multline}
        \left\Vert\hat{G}_{c_{a}\mid\mathbf{c}}\left(\mathcal{T}_{c_{a}\mid\mathbf{c}}\right) - G_{c_{a}\mid\mathbf{c}}\right\Vert_{V_a\left(\mathcal{T}_{c_{a}\mid\mathbf{c}}\right)} \leq \left\Vert\lambda G_{c_{a}\mid\mathbf{c}} \right\Vert_{V_a\left(\mathcal{T}_{c_{a}\mid\mathbf{c}}\right)^{-1}} + \left\Vert\mathbf{Z}_{\mathcal{T}_{c_{A_t}\mid\mathbf{c}}}\mathbf{B}_{\mathcal{T}_{c_{A_t}\mid\mathbf{c}}}^\top \right\Vert_{V_a\left(\mathcal{T}_{c_{a}\mid\mathbf{c}}\right)^{-1}}  \\ + \left\Vert \mathbf{Z}_{\mathcal{T}_{c_{A_t}\mid\mathbf{c}}}\mathbf{E}_{\mathcal{T}_{c_{A_t}\mid\mathbf{c}}}^\top \right\Vert_{V_a\left(\mathcal{T}_{c_{a}\mid\mathbf{c}}\right)^{-1}} \nonumber. 
    \end{multline}
    
    As it is shown in Lemma \ref{lemma:bias_error}, inequality \eqref{eq:bias_bound_matrix} is satisfied with a probability of at least $1-\delta$. For the term $\varepsilon_{a;\tau}$, since $\varepsilon_{a;\tau}$ is conditionally $B_R$-sub-Gaussian on $\mathcal{F}_{t-1}$ and $\Xi_{t}$ is $\mathcal{F}_{t-1}$ measurable, then the conditions described in \cite{NIPS2011_e1d5be1c}, Theorem 1, are satisfied. Therefore, the following inequality is satisfied with a probability of at least $1-\delta_e$:
    \begin{equation}
        \left\Vert \mathbf{Z}_{\mathcal{T}_{c_{A_t}\mid\mathbf{c}}}\mathbf{E}_{\mathcal{T}_{c_{A_t}\mid\mathbf{c}}}^\top \right\Vert_{V_a\left(\mathcal{T}_{c_{a}\mid\mathbf{c}}\right)^{-1}} \leq \sqrt{2B_R^2\log\left(\frac{1}{\delta_e}\frac{ \det(V_a\left(\mathcal{T}_{c_{a}\mid\mathbf{c}}\right))^{1/2}}{\det(\lambda I)^{1/2}}\right)}. \nonumber
    \end{equation}
    
    Finally, using Assumption \ref{assum:G_a_Assumption}, the term $\lambda G_{c_{a}\mid\mathbf{c}}$ can be bounded using the Cauchy-Schwarz inequality 
    \begin{align}
        \left\Vert \lambda G_{c_{a}\mid\mathbf{c}} \right\Vert_{V_a\left(\mathcal{T}_{c_{a}\mid\mathbf{c}}\right)^{-1}} & \leq \left\vert \lambda \right\vert \sqrt{\mbox{tr}\left(V_a\left(\mathcal{T}_{c_{a}\mid\mathbf{c}}\right)^{-1}\right)} \left\Vert G_{c_{a}\mid\mathbf{c}} \right\Vert_2 \nonumber \\
        & \leq \left\vert \lambda \right\vert \sqrt{\mbox{tr}\left(V_a\left(\mathcal{T}_{c_{a}\mid\mathbf{c}}\right)^{-1}\right)} B_G \nonumber
    \end{align}
    
    Therefore, the bound \eqref{eq:probabilistic_bound} is obtained. 
\end{proof}

Theorem \ref{theorem:model_error} states that with a probability of at least $\left(1-\delta_b\right)\left(1-\delta_e\right)$, the weighted-norm in \eqref{eq:probabilistic_bound} is upper-bounded by two terms, \eqref{eq:e_probabilistic_bound} and \eqref{eq:b_probabilistic_bound}. Note that the terms $(\delta_e,\delta_b)$ are within the interval $(0,1)$, where the closer $(\delta_e,\delta_b)$ are to $0$, the larger the terms \eqref{eq:e_probabilistic_bound} and \eqref{eq:b_probabilistic_bound} are. The opposite occurs if $(\delta_e,\delta_b)$ are closer to $1$. 

\subsection{Prediction Error}\label{sec:perturb_error}

Using Theorem \ref{theorem:model_error}, we can prove a bound on the prediction error \eqref{eq:prediction_error}. 

\begin{theorem}\label{lemma:sigma_bound}
    With a probability of at least $\left(1-\delta_e\right)\left(1-\delta_b\right)$, the following inequality is satisfied
    \begin{equation}\label{eq:difference_guarantee}
        \hat{G}_a\left(\mathcal{T}_{c_{a}\mid\mathbf{c}}\right)^\top \Xi\left(\mathbf{c}\right) - G_{c_{a}\mid\mathbf{c}}^\top \Xi\left(\mathbf{c}\right) \leq \sigma\left(\delta_e,\delta_b,c_a\mid\mathbf{c}\right),
    \end{equation}
    where $\sigma\left(\delta_e,\delta_b,c_a\mid\mathbf{c}\right)$ is defined to be 
    \begin{equation}\label{eq:sigma_def}
        \sigma\left(\delta_e,\delta_b,c_a\mid\mathbf{c}\right) \triangleq \left(e_{c_{a} \mid \mathbf{c}}\left(\delta_e\right) + b_{c_{a}\mid \mathbf{c}}\left(\delta_b\right) \right) \sqrt{\Xi_t\left(\mathbf{c}\right)^\top V_{a}\left(\mathcal{T}_{c_{a}\mid \mathbf{c}}\right)^{-1}\Xi_t\left(\mathbf{c}\right)}. 
    \end{equation}
\end{theorem}

\begin{proof}
    Let there be the following optimization problem 
    \begin{equation}\label{eq:optimization_perturb}
        \begin{array}{cc}
            \underset{a \in [k]}{\arg\max} & \tilde{G}_{c_{a}\mid\mathbf{c}}^\top \Xi\left(\mathbf{c}\right) \\
            \mbox{s.t.} & \left\Vert \hat{G}_a\left(\mathcal{T}_{c_{a}\mid\mathbf{c}}\right) - \tilde{G}_{c_{a}\mid\mathbf{c}} \right\Vert_{V_a\left(\mathcal{T}_{c_{a}\mid\mathbf{c}}\right)}^2 \leq \left(e_{c_a \mid \mathbf{c}}\left(\delta_e\right) + b_{c_a \mid \mathbf{c}}\left(\delta_b\right)\right)^2
        \end{array}.
    \end{equation}

    To solve \eqref{eq:optimization_perturb}, we have the following Lagrangian
    \begin{equation}
        L\left(\tilde{G}_{c_{a}\mid\mathbf{c}},\alpha\right) = \tilde{G}_{c_{a}\mid\mathbf{c}}^\top \Xi\left(\mathbf{c}\right) - \frac{\alpha}{2} \left(\left\Vert \hat{G}_a\left(\mathcal{T}_{c_{a}\mid\mathbf{c}}\right) - \tilde{G}_{c_{a}\mid\mathbf{c}} \right\Vert_{V_a\left(\mathcal{T}_{c_{a}\mid\mathbf{c}}\right)}^2 - \left(e_{c_a \mid \mathbf{c}}\left(\delta_e\right) + b_{c_a \mid \mathbf{c}}\left(\delta_b\right)\right)^2\right) \nonumber,
    \end{equation}

    \begin{multline}
        L\left(\tilde{G}_{c_{a}\mid\mathbf{c}},\alpha\right) = \tilde{G}_{c_{a}\mid\mathbf{c}}^\top \Xi\left(\mathbf{c}\right) \\ - \frac{\alpha}{2} \left(\left( \hat{G}_a\left(\mathcal{T}_{c_{a}\mid\mathbf{c}}\right) - \tilde{G}_{c_{a}\mid\mathbf{c}} \right)^\top V_a\left(\mathcal{T}_{c_{a}\mid\mathbf{c}}\right)\left( \hat{G}_a\left(\mathcal{T}_{c_{a}\mid\mathbf{c}}\right) - \tilde{G}_{c_{a}\mid\mathbf{c}} \right) - \left(e_{c_a \mid \mathbf{c}}\left(\delta_e\right) + b_{c_a \mid \mathbf{c}}\left(\delta_b\right)\right)^2\right) \nonumber,
    \end{multline}

    \begin{multline}
        L\left(\tilde{G}_{c_{a}\mid\mathbf{c}},\alpha\right) = \tilde{G}_{c_{a}\mid\mathbf{c}}^\top \Xi\left(\mathbf{c}\right) - \frac{\alpha}{2}\hat{G}_a\left(\mathcal{T}_{c_{a}\mid\mathbf{c}}\right)^\top V_a\left(\mathcal{T}_{c_{a}\mid\mathbf{c}}\right)\hat{G}_a\left(\mathcal{T}_{c_{a}\mid\mathbf{c}}\right) \\ + \alpha \hat{G}_a\left(\mathcal{T}_{c_{a}\mid\mathbf{c}}\right)^\top V_a\left(\mathcal{T}_{c_{a}\mid\mathbf{c}}\right) \tilde{G}_{c_{a}\mid\mathbf{c}} - 
        \frac{\alpha}{2}\tilde{G}_{c_{a}\mid\mathbf{c}}^\top V_a\left(\mathcal{T}_{c_{a}\mid\mathbf{c}}\right)\tilde{G}_{c_{a}\mid\mathbf{c}} + \frac{\alpha}{2} \left(e_{c_a \mid \mathbf{c}}\left(\delta_e\right) + b_{c_a \mid \mathbf{c}}\left(\delta_b\right)\right)^2\label{eq:lagrange_2}.
    \end{multline}

    Therefore, the partial derivatives of the Lagrangian are 
    \begin{align}
        \frac{\partial L\left(\tilde{G}_{c_{a}\mid\mathbf{c}},\alpha\right)}{\partial \tilde{G}_{c_{a}\mid\mathbf{c}}} & = \Xi\left(\mathbf{c}\right)^\top - \alpha \tilde{G}_{c_{a}\mid\mathbf{c}}^\top V_a\left(\mathcal{T}_{c_{a}\mid\mathbf{c}}\right) + \alpha \hat{G}_a\left(\mathcal{T}_{c_{a}\mid\mathbf{c}}\right)^\top V_a\left(\mathcal{T}_{c_{a}\mid\mathbf{c}}\right) \nonumber\\
        \frac{\partial L\left(\tilde{G}_{c_{a}\mid\mathbf{c}},\alpha\right)}{\partial \alpha} & = -\frac{1}{2}\left(\left\Vert \hat{G}_a\left(\mathcal{T}_{c_{a}\mid\mathbf{c}}\right) - \tilde{G}_{c_{a}\mid\mathbf{c}} \right\Vert_{V_a\left(\mathcal{T}_{c_{a}\mid\mathbf{c}}\right)}^2 + \left(e_{c_a \mid \mathbf{c}}\left(\delta_e\right) + b_{c_a \mid \mathbf{c}}\left(\delta_b\right)\right)^2\right)  \nonumber.
    \end{align}

    Setting $\partial L\left(\tilde{G}_{c_{a}\mid\mathbf{c}},\alpha\right)/\partial \tilde{G}_{c_{a}\mid\mathbf{c}} = 0$ and solving $\tilde{G}_{c_{a}\mid\mathbf{c}}$ provides 
    \begin{align}
        \Xi\left(\mathbf{c}\right)^\top - \alpha \tilde{G}_{c_{a}\mid\mathbf{c}}^\top V_a\left(\mathcal{T}_{c_{a}\mid\mathbf{c}}\right) + \alpha \hat{G}_a\left(\mathcal{T}_{c_{a}\mid\mathbf{c}}\right)^\top V_a\left(\mathcal{T}_{c_{a}\mid\mathbf{c}}\right) & = 0  \nonumber\\
        \Xi\left(\mathbf{c}\right)^\top + \alpha \hat{G}_a\left(\mathcal{T}_{c_{a}\mid\mathbf{c}}\right)^\top V_a\left(\mathcal{T}_{c_{a}\mid\mathbf{c}}\right) & = \alpha \tilde{G}_{c_{a}\mid\mathbf{c}}^\top V_a\left(\mathcal{T}_{c_{a}\mid\mathbf{c}}\right)    \nonumber  ,
    \end{align}
    \begin{equation}\label{eq:tilde_g}
        \Rightarrow \tilde{G}_{c_{a}\mid\mathbf{c}}^\top = \Xi\left(\mathbf{c}\right)^\top \frac{V_a\left(\mathcal{T}_{c_{a}\mid\mathbf{c}}\right)^{-1}}{\alpha} + \hat{G}_a\left(\mathcal{T}_{c_{a}\mid\mathbf{c}}\right)^\top . 
    \end{equation}

    Plugging in $\tilde{G}_{c_{a}\mid\mathbf{c}}$ in \eqref{eq:lagrange_2} provides 
    \begin{align}
        \frac{1}{2}\left(\left\Vert \hat{G}_a\left(\mathcal{T}_{c_{a}\mid\mathbf{c}}\right) - \left(\hat{G}_a\left(\mathcal{T}_{c_{a}\mid\mathbf{c}}\right)-\frac{V_a\left(\mathcal{T}_{c_{a}\mid\mathbf{c}}\right)^{-1}}{\alpha}\Xi\left(\mathbf{c}\right)\right)\right\Vert_{V_a\left(\mathcal{T}_{c_{a}\mid\mathbf{c}}\right)}^2 - \left(e_{c_a \mid \mathbf{c}}\left(\delta_e\right) + b_{c_a \mid \mathbf{c}}\left(\delta_b\right)\right)^2\right)& = 0 \nonumber\\ 
        \frac{1}{2}\left(\left\Vert \frac{V_a\left(\mathcal{T}_{c_{a}\mid\mathbf{c}}\right)^{-1}}{\alpha} \Xi\left(\mathbf{c}\right)\right\Vert_{V_a\left(\mathcal{T}_{c_{a}\mid\mathbf{c}}\right)}^2 - \left(e_{c_a \mid \mathbf{c}}\left(\delta_e\right) + b_{c_a \mid \mathbf{c}}\left(\delta_b\right)\right)^2\right) & = 0 \nonumber\\        
        \frac{1}{2}\left(\frac{\Xi\left(\mathbf{c}\right)^\top V_a\left(\mathcal{T}_{c_{a}\mid\mathbf{c}}\right)^{-1}V_a\left(\mathcal{T}_{c_{a}\mid\mathbf{c}}\right)V_a\left(\mathcal{T}_{c_{a}\mid\mathbf{c}}\right)^{-1}\Xi\left(\mathbf{c}\right)}{\alpha^2} - \left(e_{c_a \mid \mathbf{c}}\left(\delta_e\right) + b_{c_a \mid \mathbf{c}}\left(\delta_b\right)\right)^2\right) & = 0 \nonumber \\
        \frac{1}{2}\left(\frac{\Xi\left(\mathbf{c}\right)^\top V_a\left(\mathcal{T}_{c_{a}\mid\mathbf{c}}\right)^{-1}\Xi\left(\mathbf{c}\right)}{\alpha^2} - \left(e_{c_a \mid \mathbf{c}}\left(\delta_e\right) + b_{c_a \mid \mathbf{c}}\left(\delta_b\right)\right)^2\right) & = 0 \nonumber, 
    \end{align}
    \begin{align}
        \frac{\Xi\left(\mathbf{c}\right)^\top V_a\left(\mathcal{T}_{c_{a}\mid\mathbf{c}}\right)^{-1}\Xi\left(\mathbf{c}\right)}{\alpha^2} & = \left(e_{c_a \mid \mathbf{c}}\left(\delta_e\right) + b_{c_a \mid \mathbf{c}}\left(\delta_b\right)\right)^2 \nonumber \\
        \frac{\Xi\left(\mathbf{c}\right)^\top V_a\left(\mathcal{T}_{c_{a}\mid\mathbf{c}}\right)^{-1}\Xi\left(\mathbf{c}\right)}{\left(e_{c_a \mid \mathbf{c}}\left(\delta_e\right) + b_{c_a \mid \mathbf{c}}\left(\delta_b\right)\right)^2} & = \alpha^2 \nonumber,
    \end{align}
    \begin{equation}\label{eq:alpha}
        \Rightarrow \alpha = \sqrt{\frac{\Xi\left(\mathbf{c}\right)^\top V_a\left(\mathcal{T}_{c_{a}\mid\mathbf{c}}\right)^{-1}\Xi\left(\mathbf{c}\right)}{\left(e_{c_a \mid \mathbf{c}}\left(\delta_e\right) + b_{c_a \mid \mathbf{c}}\left(\delta_b\right)\right)^2}}. 
    \end{equation}

    Finally, plugging in \eqref{eq:alpha} and \eqref{eq:tilde_g} into $\tilde{G}_{c_{a}\mid\mathbf{c}}^\top \Xi\left(\mathbf{c}\right)$ provides
    \begin{equation}
        \tilde{G}_{c_{a}\mid\mathbf{c}}^\top \Xi\left(\mathbf{c}\right)
        = \hat{G}_a\left(\mathcal{T}_{c_{a}\mid\mathbf{c}}\right)^\top \Xi\left(\mathbf{c}\right) +\Xi\left(\mathbf{c}\right)^\top V_a\left(\mathcal{T}_{c_{a}\mid\mathbf{c}}\right)^{-1}\Xi\left(\mathbf{c}\right)\sqrt{\frac{\left(e_{c_a \mid \mathbf{c}}\left(\delta_e\right) + b_{c_a \mid \mathbf{c}}\left(\delta_b\right)\right)^2}{\Xi\left(\mathbf{c}\right)^\top V_a\left(\mathcal{T}_{c_{a}\mid\mathbf{c}}\right)^{-1}\Xi\left(\mathbf{c}\right)}} \nonumber. 
    \end{equation}
    
    The solution therefore for \eqref{eq:optimization_perturb} is 
    \begin{equation}\label{eq:maximum value}
        \hat{G}_a\left(\mathcal{T}_{c_{a}\mid\mathbf{c}}\right)^\top \Xi\left(\mathbf{c}\right) + \left(e_{c_a \mid \mathbf{c}}\left(\delta_e\right) + b_{c_a \mid \mathbf{c}}\left(\delta_b\right)\right)\sqrt{\Xi\left(\mathbf{c}\right)^\top V_a\left(\mathcal{T}_{c_{a}\mid\mathbf{c}}\right)^{-1}\Xi\left(\mathbf{c}\right)}. 
    \end{equation}

    If we define $\sigma\left(\delta_e,\delta_b,c_a\mid\mathbf{c}\right)$ to be \eqref{eq:sigma_def}, then the inequality \eqref{eq:difference_guarantee} is satisfied with a probability of at least $\left(1-\delta_e\right)\left(1-\delta_b\right)$. 
\end{proof}

The following lemma is provided for proving the bound of \eqref{eq:sigma_def}. 

\begin{lemma}\label{lemma:perturb_bound}
    The upper bound of $\sigma\left(\delta_e,\delta_b,c_a\mid\mathbf{c}\right)$ is satisfied with a probability of at least $1-\delta$ where the bound is 
    \begin{equation}\label{eq:sigma_bound}
        \sigma\left(\delta_e,\delta_b,c_a\mid\mathbf{c}\right) \leq B\left(\delta\mid\mathbf{c}\right),
    \end{equation}
    where $B\left(\delta\mid\mathbf{c}\right)$ is defined to be \eqref{eq:B_def}. 
\end{lemma}

\begin{proof}
    The upper bound of $\sigma\left(\delta_e,\delta_b,c_a\mid\mathbf{c}\right)$ is as follows: given that the following inequality is satisfied with a probability of least $1-\delta$:
    \begin{equation}\label{eq:markov_xi}
        \left\Vert \Xi_t\left(\mathbf{c}\right) \right\Vert_2 \leq \frac{\mathbb{E}\left[\left\Vert \Xi_t\left(\mathbf{c}\right) \right\Vert_2\right]}{\delta}, 
    \end{equation}
    then according to Lemma 11 in \cite{NIPS2011_e1d5be1c}, $e_{c_a\mid \mathbf{c}}\left(\delta_e\right)$ is upper bounded as (with a probability of at least $1-\delta$)
    \begin{align}
        e_{c_a\mid \mathbf{c}}\left(\delta_e\right) & \leq \sqrt{2B_R^2 \log\left(\frac{1}{\delta_e}\frac{\left(\mbox{tr}\left(\lambda I_s\right) + \left\vert \mathcal{T}_{c_a\mid \mathbf{c}}\right\vert \frac{\mathbb{E}\left[\left\Vert \Xi_t\left(\mathbf{c}\right) \right\Vert_2\right]}{\delta}\right)^{ s/2}}{\sqrt{\det\left(\lambda I_s\right)}}\right)} \nonumber \\
        & \leq \sqrt{2B_R^2 \log\left(\frac{1}{\delta_e}\frac{\left( s\lambda + (n-s) \frac{\mathbb{E}\left[\left\Vert \Xi_t\left(\mathbf{c}\right) \right\Vert_2\right]}{\delta}\right)^{ s/2}}{\lambda^{s/2}}\right)} \nonumber. 
    \end{align}

    The term $b_{c_a\mid \mathbf{c}}\left(\delta_b\right)$ is based on bound \eqref{eq:bias_bound_matrix}. To bound $\sqrt{\mbox{tr}\left(I_s - \lambda V_a\left(\mathcal{T}_{c_{a}\mid\mathbf{c}}\right)^{-1}\right)}$ in \eqref{eq:bias_bound_matrix}, note that since $V_a\left(\mathcal{T}_{c_{a}\mid\mathbf{c}}\right) \succeq \lambda I$, then
    \begin{equation}
        \sqrt{\mbox{tr}\left(I_s - \lambda V_a\left(\mathcal{T}_{c_{a}\mid\mathbf{c}}\right)^{-1}\right)} \leq \sqrt{\mbox{tr}\left(I_s\right)} = \sqrt{s}\nonumber.
    \end{equation}
    
    Therefore, we have the upper bound of $b_{c_a\mid \mathbf{c}}\left(\delta_b\right)$:
    \begin{equation}
        b_{c_a\mid \mathbf{c}}\left(\delta_b\right) \leq \sqrt{n-s} \frac{B_c B_R}{\delta_b} \sqrt{s} \nonumber. 
    \end{equation}
        
    Finally, the term $\sqrt{\Xi_t\left(\mathbf{c}\right)^\top V_a\left(\mathcal{T}_{c_a\mid \mathbf{c}}\right)^{-1}\Xi_t\left(\mathbf{c}\right)}$ has the following bound with a probability of at least $1-\delta$
    \begin{align}
        \sqrt{\Xi_t\left(\mathbf{c}\right)^\top V_a\left(\mathcal{T}_{c_a\mid \mathbf{c}}\right)^{-1}\Xi_t\left(\mathbf{c}\right)}& \leq \sqrt{\mbox{tr}\left(V_a\left(\mathcal{T}_{c_a\mid \mathbf{c}}\right)^{-1}\right)} \left\Vert \Xi_t\left(\mathbf{c}\right) \right\Vert_2 \nonumber \\
        & \leq \sqrt{\mbox{tr}\left(\frac{I_s}{\lambda}\right)} \left\Vert \Xi_t\left(\mathbf{c}\right) \right\Vert_2 \nonumber,
    \end{align}
    \begin{equation}
        \Rightarrow  \sqrt{\Xi_t\left(\mathbf{c}\right)^\top V_a\left(\mathcal{T}_{c_a\mid \mathbf{c}}\right)^{-1}\Xi_t\left(\mathbf{c}\right)} \leq \sqrt{\frac{s}{\lambda}} \frac{\mathbb{E}\left[\left\Vert \Xi_t\left(\mathbf{c}\right) \right\Vert_2\right]}{\delta} \label{eq:sigma_final_inequality}. 
    \end{equation}
    where the final inequality \eqref{eq:sigma_final_inequality} uses \eqref{eq:markov_xi}. Therefore, the bound for $\sigma\left(\delta_e,\delta_b,c_a\mid\mathbf{c}\right)$ is \eqref{eq:sigma_bound}, which satisfied with a probability of at least $1-\delta$.
\end{proof}

\subsection{Regret Performance}\label{appendix:regret_performance}

The following theorem provides a bound for regret. 

\begin{theorem}\label{theorem:regret_performance}
    Let $\delta_e = \delta_b = \delta \in (0,1)$. Regret satisfies the following inequality with a probability of at least $\left(1-\delta\right)^4$
    \begin{equation}\label{eq:regret_bound}
        R_n \leq \sum_{a = 1}^k 2nB_c^2 B_R^2 \left(1-\left(1-\delta\right)^4\left(1-\exp\left(\frac{-\phi_{a,a_t^*}^2}{2\Delta G_{c_a \mid \mathbf{c}}^\top \Sigma_{\Xi_t\left(\mathbf{c}\right)}\Delta G_{c_a \mid \mathbf{c}}}\right)\right) \right),
    \end{equation}
    where $a_t^* \in [k]$ is the index of the optimal action $c_{a_t^*} \in \mathcal{A}$ at round $t$. Terms $\Delta G_{c_a \mid \mathbf{c}}$, $\phi_{a,a_t^*}^2$, and $\Sigma_{\Xi_t\left(\mathbf{c}\right)}$ are defined to be 
    \begin{align}
        \Delta G_{c_a \mid \mathbf{c}} & \triangleq G_{c_{a^*} \mid \mathbf{c}} - G_{c_{a} \mid \mathbf{c}} \label{eq:delta_G_definition} \\
        \phi_{a,a^*}^2 & \triangleq u_{a^*} - u_{a} + \sigma\left(\delta,\delta,c_{a^*}\mid\mathbf{c}\right) + \sigma\left(\delta,\delta,c_a\mid\mathbf{c}\right) \label{eq:phi_definition} \\
        \Sigma_{\Xi_t\left(\mathbf{c}\right)} & \triangleq \mathbb{E}\left[\Xi_t\left(\mathbf{c}\right)\Xi_t\left(\mathbf{c}\right)^\top\right] \label{eq:xi_covariance},
    \end{align}
    and $\sigma\left(\delta,\delta,c_a\mid\mathbf{c}\right)$ is defined to be \eqref{eq:sigma_def}. 
\end{theorem}

\begin{proof}
    Define $c_{a^*} \in \mathcal{A}$ as the optimal action, i.e. choosing action $c_{a^*} \in \mathcal{A}$ samples a reward $X_t^*$ from \eqref{eq:linear_dynamical_system} such that for any action $c_a \in \mathcal{A}$ with a sampled reward $X_t$, the reward $X_t^* \geq X_t$. The probability of choosing the optimal action $c_{a^*} \in \mathcal{A}$ is based on the following event
    \begin{equation}
        \hat{G}_{c_{a^*}\mid \mathbf{c}}\left(\mathcal{T}_{c_{a^*}\mid\mathbf{c}}\right)^\top \Xi\left(\mathbf{c}\right) + u_{a^*} \geq \hat{G}_{c_a\mid \mathbf{c}}\left(\mathcal{T}_{c_a\mid\mathbf{c}}\right)^\top \Xi\left(\mathbf{c}\right) + u_{a} \nonumber. 
    \end{equation}

    With a probability of at least $\left(1-\delta\right)^4$, the following inequality is satisfied 
    \begin{equation}
        G_{c_{a^*}\mid \mathbf{c}}^\top \Xi\left(\mathbf{c}\right) + u_{a^*} + \sigma\left(\delta,\delta,c_{a^*}\mid\mathbf{c}\right) \geq G_{c_a\mid \mathbf{c}}^\top \Xi\left(\mathbf{c}\right) + u_a - \sigma\left(\delta,\delta,c_a\mid\mathbf{c}\right) \nonumber.
    \end{equation}

    Rearranging provides the following    
    \begin{equation}\label{eq:event_markov_inequality}
        \left(G_{c_{a^*}\mid \mathbf{c}} - G_{c_a\mid \mathbf{c}}\right)^\top \Xi\left(\mathbf{c}\right) \geq u_a - \sigma\left(\delta,\delta,c_a\mid\mathbf{c}\right) - u_{a^*} - \sigma\left(\delta,\delta,c_{a^*}\mid\mathbf{c}\right)\nonumber.
    \end{equation}
    
    Define the left side of the inequality as \eqref{eq:delta_G_definition} and the right side of inequality \eqref{eq:event_markov_inequality} as \eqref{eq:phi_definition}. Since $\Delta G_{c_a \mid \mathbf{c}}^\top \left(\Xi\left(\mathbf{c}\right) - \mathbb{E}\left[\Xi\left(\mathbf{c}\right)\right]\right)$ is sub-Gaussian, then using the Cramer-Chernoff bound (based on a Corollary 5.5 in \cite{lattimore2020bandit}) provides
    \begin{equation}
        \mathbb{P}\left(\Delta G_{c_a \mid \mathbf{c}}^\top \Xi\left(\mathbf{c}\right) \geq \Delta G_{c_a \mid \mathbf{c}}^\top \mathbb{E}\left[\Xi\left(\mathbf{c}\right)\right] - \phi_{a,a^*}\right) \geq 1-\exp\left(\frac{-\phi_{a,a^*}^2}{2\Delta G_{c_a \mid \mathbf{c}}^\top \Sigma_{\Xi_t\left(\mathbf{c}\right)}\Delta G_{c_a \mid \mathbf{c}}}\right)
    \end{equation}
    the optimal action is chosen. The probability of choosing the sub-optimal action is the probability of not choosing the optimal action. This implies that with a probability of at most 
    \begin{equation}
        p_a \triangleq 1-\left(1-\delta\right)^4\left(1-\exp\left(\frac{-\phi_{a,a^*}^2}{2\Delta G_{c_a \mid \mathbf{c}}^\top \Sigma_{\Xi_t\left(\mathbf{c}\right)}\Delta G_{c_a \mid \mathbf{c}}}\right)\right) \nonumber, 
    \end{equation}
    the sub-optimal action is chosen. Using the law of iterated expectations \cite{wooldridge2010econometric}, the upper bound for instantaneous regret $\mathbb{E}\left[X_t^* - X_t\right]$ is 
    \begin{equation}
        \mathbb{E}\left[X_t^* - X_t\right] = \sum_{a=1}^k \mathbb{E}\left[\left\langle \Delta c_a, z_t \right\rangle \mid a \right]\mathbb{P}\left[a\right] \leq \sum_{a=1}^k \mathbb{E}\left[\left\langle \Delta c_a, z_t \right\rangle \mid a \right]p_a \nonumber, 
    \end{equation}
    \begin{equation}
        \Rightarrow \mathbb{E}\left[X_t^* - X_t\right] \leq \sum_{a = 1}^k 2B_c^2 B_R^2 \left(1-\left(1-\delta\right)^4\left(1-\exp\left(\frac{-\phi_{a,a^*}^2}{2\Delta G_{c_a \mid \mathbf{c}}^\top \Sigma_{\Xi_t\left(\mathbf{c}\right)}\Delta G_{c_a \mid \mathbf{c}}}\right)\right) \right) \nonumber. 
    \end{equation}

    Therefore, regret is upper-bounded by inequality \eqref{eq:regret_bound}.     
\end{proof}

The regret bound is a function of prediction errors $\sigma\left(\delta,\delta,c_a\mid\mathbf{c}\right)$ and $\sigma\left(\delta,\delta,c_{a^*}\mid\mathbf{c}\right)$. To control the impact of prediction error $\sigma\left(\delta,\delta,c_a\mid\mathbf{c}\right)$, we set $u_{a}$ and $u_{a^*}$ to be $u_a = \sigma\left(\delta,\delta,c_a\mid\mathbf{c}\right)$ and $u_{a^*} = \sigma\left(\delta,\delta,c_{a^*}\mid\mathbf{c}\right)$. By setting $u_a = \sigma\left(\delta,\delta,c_a\mid\mathbf{c}\right)$ for all $c_a \in \mathcal{A}$, setting $\delta_e = \delta_b = \delta \in (0,1)$, and using Lemma \ref{lemma:sigma_bound} and Theorem \ref{theorem:regret_performance}, Theorem \ref{theorem:algo_performance} can be easily proven.  

\end{document}